\documentclass[pdflatex]{sn-jnl}
\usepackage{silence}
\usepackage{graphicx}%
\usepackage{multirow}%
\usepackage{amsmath,amssymb,amsfonts}%
\usepackage{amsthm}%
\usepackage{mathrsfs}%
\usepackage[title]{appendix}%
\usepackage{xcolor}%
\usepackage{textcomp}%
\usepackage{manyfoot}%
\usepackage{booktabs}%
\usepackage{algorithm}%
\usepackage{algorithmicx}%
\usepackage{algpseudocode}%
\usepackage{listings}%
\usepackage{caption}
\usepackage{booktabs}
\usepackage{subcaption}
\usepackage{xr}
\usepackage{siunitx}
\usepackage[export]{adjustbox}
\usepackage{threeparttable}
\usepackage[backend=biber, style=numeric]{biblatex}
\usepackage{hyperref}  
\usepackage{xr-hyper} 

\theoremstyle{thmstyleone}%
\newtheorem{theorem}{Theorem}%

\theoremstyle{thmstyletwo}%

\theoremstyle{thmstylethree}%

\newcommand{\beginsupplement}{%
  \setcounter{section}{0}
  \setcounter{subsection}{0}
  \setcounter{table}{0}
  \setcounter{figure}{0}
  \setcounter{equation}{0}

  \renewcommand{\thesection}{S\arabic{section}}
  \renewcommand{\thesubsection}{S\arabic{section}.\arabic{subsection}}
  \renewcommand{\thetable}{S\arabic{table}}
  \renewcommand{\thefigure}{S\arabic{figure}}
  \renewcommand{\theequation}{S\arabic{equation}}
}

\begin{document}

\begin{refsection}
\title[]{From Vision to Decision: Neuromorphic Control for Autonomous Navigation and Tracking}

\author*[1]{\fnm{Chuwei} \sur{Wang}}\email{cw996@cornell.edu}

\author[2]{\fnm{Eduardo} \sur{Sebastián}}\email{es2121@cam.ac.uk}

\author[2]{\fnm{Amanda} \sur{Prorok}}\email{asp45@cam.ac.uk}

\author*[1]{\fnm{Anastasia} \sur{Bizyaeva}}\email{anastasiab@cornell.edu}

\affil*[1]{\orgdiv{Sibley School of Mechanical and Aerospace Engineering}, \orgname{Cornell University}, \orgaddress{\city{Ithaca}, \postcode{14853}, \state{New York}, \country{USA}}}

\affil[2]{\orgdiv{Department of Computer Science and Technology}, \orgname{University of Cambridge}, \orgaddress{\city{Cambridge}, \postcode{CB3 0FD}, \country{United Kingdom}}}

\abstract{Robotic navigation has historically struggled to reconcile reactive, sensor-based control with the decisive capabilities of model-based planners. This duality becomes critical when the absence of a predominant option among goals leads to indecision, challenging reactive systems to break symmetries without computationally-intense planners. We propose a parsimonious neuromorphic control framework that bridges this gap for vision-guided navigation and tracking. Image pixels from an onboard camera are encoded as inputs to dynamic neuronal populations that directly transform visual target excitation into egocentric motion commands. A dynamic bifurcation mechanism resolves indecision by delaying commitment until a critical point induced by the environmental geometry. Inspired by recently proposed mechanistic models of animal cognition and opinion dynamics, the neuromorphic controller provides real-time autonomy with a minimal computational burden, a small number of interpretable parameters, and can be seamlessly integrated with application-specific image processing pipelines. We validate our approach in simulation environments as well as on an experimental quadrotor platform.}


\maketitle

\section{Introduction}\label{sec1}

Autonomous multi-goal navigation and tracking have long been challenged by the architectural tension between distal and proximal approaches \cite{harnad1990symbol}. The distal perspective focuses on the object-centric layout of the environment, requiring the agent to maintain an internal model of objects and locations that are spatially removed from the sensors \cite{nillson1984shakey, laird1987soar}. In contrast, the proximal approach anchors the agent’s reality in the immediate and reactive sensory stimulation, coupling raw input streams directly to motor actions without an intermediate representation \cite{brooks1991intelligence, braitenberg1986vehicles}. Historically, efforts to address decision-making have centered on distal, model-based methods, which operate on a reconstructed ``world of objects'' via hierarchical perception, mapping, and planning stages \cite{Mohta18, Tordesillas19Faster, Zhou21TRO}. While interpretable, these distal architectures are computationally intensive and often suffer from a disconnect between the internal model and the immediate reality, making their performance sensitive to model mismatches. Conversely, proximal approaches rely on direct sensorimotor couplings, reacting to the ``stream of sensation'' rather than a map. However, purely proximal agents struggle with indecision \cite{koren1991potential, grover2020does, Reis21CSL, grover2023before}.  When an agent faces multiple equally or similarly viable options, like in symmetrical scenarios, immediate sensory data is ambiguous, trapping the robot in local minima or deadlocks. The challenge, therefore, is to achieve the reliable decision-making of distal systems with the efficiency of proximal methods.

In the pursuit of enhancing proximal control, recent data-driven advances have successfully reduced the dependence on high-resolution maps or precise localization, pushing the boundaries of what purely sensor-based agents can achieve and making them better suited for dynamic environments and lightweight platforms \cite{ Kaufmann23Nature, Geles2024a, verraest2025skydreamer, Pfeiffer17ICRA, Zhu17ICRA, Devo20TRO, Zhang_2025}. While these learning-based methods excel in scenarios like drone racing or end-to-end visual navigation, they typically rely on identifying statistical correlations in the training data rather than resolving the underlying geometric ambiguity. Consequently, they often lack theoretical guarantees for breaking indecision when confronted with conflicting objectives such as symmetric goals, and struggle to provide the fast decision-making required for robust real-time autonomy. In this paper, we present a reliable, efficient neuromorphic control framework for vision-guided navigation and tracking that generates interpretable proximal decision-making in such indecisive environments, achieving planner-like behavior while operating reactively on vision streams. 

The proposed control framework draws direct inspiration from the dynamic cognitive mechanisms employed by animals during navigation. In contrast to artificial systems, animals exhibit rapid, flexible decision-making when moving through space \cite{Seelig2015NeuralDF, sridhar2021geometry,oscar2023simple}. When doing so, they use locally available sensory data and are able to navigate complex, unfamiliar environments using minimal computational effort.  Over recent decades, neuroscience has made substantial progress in understanding how sensory signals relate to movement in a variety of animals, from fruit flies to vertebrates \cite{nordstrom2009feature, nicholas2021facilitation}. These insights gave rise to interpretable mathematical models of cognitive computations during spatial movement with direct implications for robotics \cite{novo2024neuromorphic}.  

The ring-attractor model \cite{Seelig2015NeuralDF, Kakaria17, Sung17sci} has emerged as particularly important for understanding spatial navigation. This model encodes a ring-shaped neuronal connectivity pattern with local excitation and global inhibition which supports the formation of stable activity ``bumps'' that can encode continuous variables such as heading direction or position. Its principles have also been applied in robotics---for example, in \cite{Rivero-Ortega23} the authors implemented a ring-attractor-based scheme to guide robots to target locations while avoiding static and dynamic obstacles, and in \cite{knowles2023ring} the authors used coupled spiking ring attractors to perform path integration and maintain an internal estimate of a robot's position. However, these theoretical frameworks typically rely on abstract state variables and do not address the fact that, in robotics, information originates from the rich stream of data from an onboard camera, requiring a dedicated processing stage to extract target evidence from raw vision.  

Beyond their representational capability, ring attractor-like structures in perceptual decision circuits may explain how animals robustly avoid symmetry-induced indecision during navigation. Recently, Sridhar et al. investigated how animals navigate towards multiple targets, such as food sources \cite{sridhar2021geometry}. The authors found that, remarkably, animals break down multi-alternative decisions into sequences of binary decisions that are characterized by sharp changes in animals' spatial trajectories observed in experiments. To explain this phenomenon mechanistically, the authors developed a statistical physics model that encodes a ring attractor-like structure and maps ensembles of spin variables to agent trajectories in space, successfully recovering the behavioral patterns observed in experiments with minimal parameter fitting. Critically, the decision points in the animals' paths correspond to phase transitions in this model, or equivalently, bifurcations in its mean field approximation \cite{gorbonos2024geometrical}. This explanation provides a principled link between the dynamics of neural circuits and discrete choice behavior in continuous space, suggesting an elegant solution to the tension between proximal and distal approaches in robotic navigation that is grounded in cognitive and dynamical principles.

In our proposed neuromorphic control framework, inspired by the insights and modeling efforts of Sridhar et al. \cite{sridhar2021geometry} as well as by recent investigation into dynamic symmetry-breaking mechanisms in multi-alternative decision-making \cite{franci2023breaking,leonard2024fast}, we leverage a simple cognitive circuit that exhibits bifurcations to resolve symmetry-induced movement indecision in autonomous navigation.  Recent literature provides strong evidence for the promise of leveraging properties of bifurcations in decision circuits for the design of flexible and robust autonomous behavior. In related work, the recently developed Nonlinear Opinion Dynamics (NOD) framework provides a complementary dynamical systems perspective on resolving indecision in multi-alternative decision-making through the lens of bifurcation theory~\cite{bizyaeva2022nonlinear,leonard2024fast}. 
Notably, NOD models are closely related in structure to the mean field model of spatial decision-making recovered in the work of Sridhar et al. \cite{sridhar2021geometry}, sharing the same fundamental symmetry-breaking mechanism. Recently the NOD framework has been successfully applied to resolve indecision across a variety of robotic applications. For instance, in \cite{Cathcart23iros}, NOD is used to model the preference for left or right in a human-robot passing problem to guarantee deadlock breaking and prevent oscillatory behavior. In \cite{Varma25ecc}, the authors combined NOD with control barrier functions to reduce deadlocks in dense environments, while \cite{Diego25AVOCADO} integrated NOD with a velocity obstacle formulation for collision avoidance in multi-agent environments. In \cite{amorim2024spatially}, a continuous-option extension of NOD with a translation-symmetric kernel reminiscent of ring attractor structure is used to decide on a navigation target from a continuum of inputs. In our control framework, analogously to NOD models, we will take a deterministic, dynamical systems modeling approach as opposed to the probabilistic, statistical physics approach of the animal behavior modeling in \cite{sridhar2021geometry,gorbonos2024geometrical}.

The primary contributions in the presented work are the following. Building on mechanistic principles of dynamic decision-making \cite{sridhar2021geometry, gorbonos2024geometrical, bizyaeva2022nonlinear, leonard2024fast}, we propose a neuromorphic control framework that effectively reconciles the proximal-distal dichotomy for vision-guided navigation.
Our contribution is the development of a perception-to-control architecture in which high-dimensional camera inputs are encoded as a population of neural activations. By embedding the decision-making bifurcation directly into the sensorimotor loop, this approach allows visual target evidence to be transformed into egocentric motion commands, granting the agent the decisiveness of a distal planner while operating purely on proximal sensory streams. This neural dynamics formulation endows the agent with rapid symmetry-breaking, robust target selection, and low-cost autonomy using only onboard visual sensing. Furthermore, this formulation enables interpretable neural dynamics models to operate on image-derived stimuli and supports plug-and-play integration of diverse image-processing front-ends. Additionally, the proposed neural dynamic control logic relies on simple continuous-time computations with saturating nonlinearities, making it compatible with future analog implementation in neuromorphic subthreshold CMOS circuits \cite{indiveri2011neuromorphic,poon2011neuromorphic,mead2012analog}. This opens up a clear pathway towards implementation of efficient embodied decision-making in neuromorphic robotic hardware as an avenue of future inquiry, complementary to recent efforts in neuromorphic approaches to autonomous navigation \cite{mitchell2017neon,abdelrahman2025neuromorphic,shangfully,novo2024neuromorphic,schoepe2024finding}. We illustrate effectiveness of our approach in multi-target navigation and tracking scenarios in simulation environments and in hardware quadrotor experiments.

The proposed neuromorphic control framework is illustrated in Fig.~\ref{fig:pipeline}. An autonomous robotic system, designed for fully onboard perception and computation, acquires RGB images in real-time. A user-defined target detector processes the incoming image and outputs a matrix where each element, associated to a pixel, is assigned a value representing the strength of target evidence. Depending on the considered visual features, the target detector can be based on color, such as Hue-Saturation-Value (HSV) masks; movement, such as optical flow filters; or more complex combinations of pre-processed features, such as those provided by YOLO modules \cite{redmon2016you}. In addition to a value of target evidence, each pixel is assigned a population of neurons, whose state evolves according to a discretization of a continuous-time dynamical system model of decision-making that describes the temporal evolution of firing rates in interconnected groups of neurons. Therefore, the combination of all pixels provides a collection of neuron populations whose 
neural activity, driven by the strength of target evidence, represents the relative preference for moving in each direction within the field of view (FOV) of the robot. A direction is a unit vector that starts at the origin of the egocentric robot body frame and points in the direction of the pixel. The velocity command, again in the egocentric robot body frame, is the combination of the directional vectors weighted by the relative preference from the neural activity, hence steering the robot in the direction of preference according to the image stimuli. Following principles of cognitive dynamics and empirical evidence on neural activity thresholding \cite{forstmann2016sequential, ratcliff2016diffusion, roxin2008neurobiological}, the computation of the velocity command only considers a subset of neural populations whose neural activity is beyond a user-defined decision threshold. Finally, a low-level controller transforms egocentric velocity into the actions executed by the robot.

\begin{figure}[htbp]
    \centering
    \includegraphics[width=0.98\linewidth]{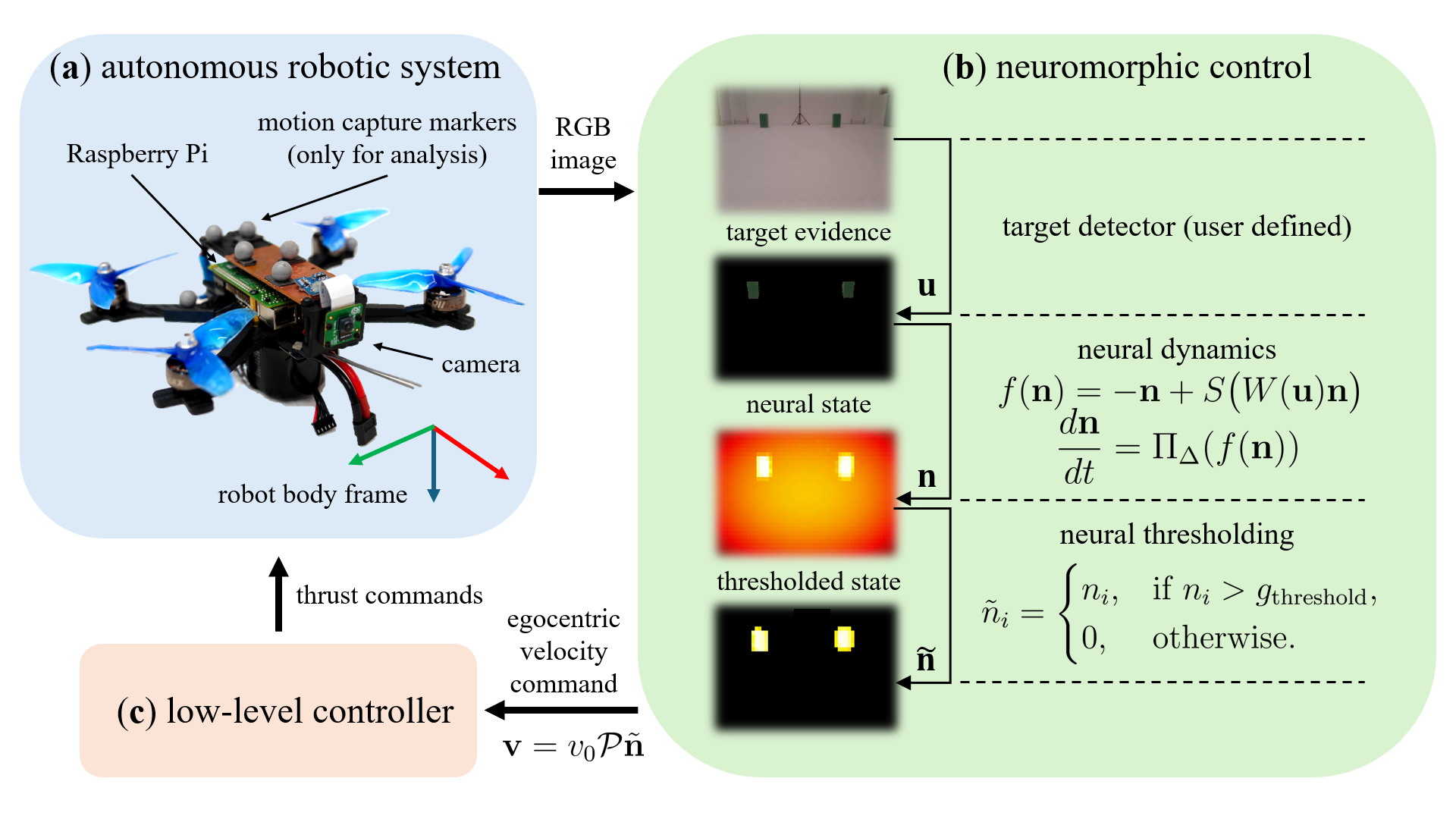}
    \caption{\textbf{Overview of the neuromorphic control framework}. \textbf{a} We consider an autonomous robotic system designed to run perception and computation onboard. In this case, a quadrotor is equipped with an RGB camera and a Raspberry Pi as its sole processing unit. The motion capture markers are included for completeness, as they are used to record data for statistical trajectory analysis only. \textbf{b} The  neuromorphic controller processes the image stream to extract relevant target features, translated into a value per pixel that models the strength of the feature. This is used as input to the neural dynamics, that evolve the relative preference over per-pixel spatial directions, yielding a velocity command in the egocentric body frame of the robot as a combination of those directions whose activity is above a threshold. \textbf{c} The velocity command is transformed into action commands that steer the robot towards the direction of preference according to the visual stimuli.}
    \label{fig:pipeline}
\end{figure}

The core component of the neuromorphic framework is the neural dynamics modeling the strength of preference for spatial directions, leading to unique properties as a direct and necessary consequence of its bio-inspired structure. Each pixel’s contribution is processed by its corresponding neural population, enabling the system to respond to the overall visual evidence without requiring object classification or tracking. The controller only needs to know which regions of the image contain potential targets, allowing a wide range of detection methods to be used. This per-pixel design naturally supports arbitrary numbers and configurations of targets without modifying the underlying neural dynamics.
The neural dynamics is not a passive filter but an active, nonlinear decision-making system. Its formulation is inspired by cognitive models of spatial choice in animals \cite{sridhar2021geometry,gorbonos2024geometrical} and collective computation with nonlinear opinion dynamics \cite{bizyaeva2022nonlinear,leonard2024fast}, which mathematically guarantee the existence of a bifurcation. This dynamic bifurcation is the central mechanism for resolving indecision. As the robot moves, it reaches a physical location that corresponds to a critical point at which its current heading decision becomes unstable, creating a spatial indecision region where the collective neural state becomes ultra-sensitive. In this state, the entire system is forced to rapidly commit to a single, stable decision state (e.g., ``go left'' vs. ``go right''). This process robustly breaks symmetries and avoids deadlocks or oscillatory states. Thus, the neuromorphic controller is characterized by a proximal model that reacts to the image stream yet has the decision-making capabilities of a planner without relying on map representations. Furthermore, the continuous-time dynamics provide an intrinsic memory term, meaning that neural activity integrates visual evidence over a short time window rather than reacting instantaneously. This temporal integration is the source of its temporal filtering capabilities, naturally filtering high-frequency sensor noise such as corrupted video frames, and allowing the robot to maintain its course even if a target is temporarily lost from view. The neural dynamics model is both parsimonious, being governed by a few interpretable parameters, and computationally lightweight, enabling real-time performance. Finally, by completely decoupling the perception front-end (which merely populates the target evidence matrix) from the core dynamics, the framework is compatible with any visual processing pipeline.

\section{Results}\label{sec2}

In this section, we showcase the performance of our neuromorphic bifurcation-guided control strategy in multi-target navigation and tracking scenarios. 
First, we contrast the neuromorphic strategy against several baseline methods in idealized two-dimensional multi-target environments with global position knowledge (Section~\ref{Sec2_1}). We show that the neuromorphic strategy yields interpretable decision-making even in environments with symmetries where proximal approaches fail, at a fraction of the computational cost of distal and end-to-end methods. 
Next, we demonstrate the decision-making bifurcation and sequential navigation capability of our neuromorphic strategy with idealized visual inputs in two-dimensional multi-target navigation scenarios (Section~\ref{Sec2_2}).
We then demonstrate the complete control pipeline in both a photorealistic simulation environment (Section~\ref{Sec2_3}) and real-world hardware experiments (Section~\ref{Sec2_4}). 

\subsection{The geometry of robot spatial decision-making}\label{Sec2_1}

We first illustrate that symmetric environments are challenging for robotic navigation, which motivates our approach. To do this, we conduct a series of simulation experiments to contrast the spatial trajectories generated by our proposed neural dynamics (ND) framework against several state-of-the-art navigation algorithms in simplified two-dimensional multi-goal environments. The construction of our simulation environments takes inspiration from the animal decision-making experiments performed by Sridhar et al. in \cite{sridhar2021geometry}. The objective is to evaluate each method's ability to make a decision and subsequently reach one of the available goals. We compared the ND framework with model predictive control (MPC), potential field (PF), and reinforcement learning (RL) baselines (see Supplementary Materials Sec.~1
for a detailed formulation). 

The robot’s configuration at time~$t$ is represented as $z_t = [x, v_x, y, v_y]^{\top}$, where $\mathbf{p}_t = [x, y]^{\top}$ denotes its position. Each target’s position is denoted by $\mathbf{g}_i = [x_{t,i}, y_{t,i}]^{\top}$. In these simulations, we assume that the robot has access to its own global position as well as the positions of all targets, i.e. a global map and accurate localization. However, the robot’s state is perturbed at every time step by additive zero-mean Gaussian noise with standard deviation $\sigma_z = 0.01$, reflecting some localization uncertainty.
Reflecting this global knowledge, the ND model includes exactly $K$ competing neural populations in environments with $K$ targets, with each neural activity variable $\bar{n}_i$ encoding the magnitude of relative preference for target $i$ over other targets. The neural activity rates $\bar{\mathbf{n}}$ are dynamically updated based on the relative target positions, for example the target matrix $\bar{\mathcal{P}} = [\bar{\mathbf{p}}_1, \bar{\mathbf{p}}_2] = \left[\frac{(\mathbf{g}_1 - \mathbf{p}_t)^{\top}}{\lVert \mathbf{g}_1 - \mathbf{p}_t \rVert}, \frac{(\mathbf{g}_2 - \mathbf{p}_t)^{\top}}{\lVert \mathbf{g}_2 - \mathbf{p}_t \rVert}\right]$ contains unit vectors pointing from the agent to targets 1 and 2,   when $K = 2$. The neural dynamics are expressed as 
\begin{equation}
\frac{d \bar{\mathbf{n}}}{dt} = \Pi_{\Delta} \big(-\bar{\mathbf{n}} + S(\bar{\mathcal{P}}^{\top}\bar{\mathcal{P}}\bar{\mathbf{n}}) \big)
\label{Eq:benchmarkmodel}
\end{equation}
where $S: \mathbb{R} \to (0,a)$ is a saturating function applied element-wise to its input vector and $\Pi_{\Delta}: \mathbb{R}_K^{+} \to \Delta_{K-1}$ maps the neural activity to the unit simplex -- see Methods (Section~\ref{sec11}) for further details. Instantaneous velocity of the agent is a weighted average of the target orientations  relative to the moving agent $\mathbf{v} = v_0 \bar{\mathcal{P}} \bar{\mathbf{n}} $.

For MPC, we implemented two variants: a weighted-sum cost formulation and a soft-minimum formulation designed to mitigate indecision among multiple targets. For RL, we trained a multi-goal navigation policy with different parameterizations of the maximum speed, acceleration and observation structure and then performed a statistical analysis to ensure the selected policy was not basing its decision on any overlooked bias induced during training (e.g., the robot going always to the first target appearing in its observation vector instead of going to the closest one). The chosen policy was deployed with a different order of observations for the targets, allowing us to examine how goal-sequence consistency influences learned decision-making behavior.  

We tested MPC1 (Weighted-sum), MPC2 (Soft-min), PF, RL and ND methods in three different geometric target settings as shown in Fig.~\ref{fig:benchmark}. MPC1 and PF failed to complete the task in two of the environments (first and third row), as both were trapped in local minima and could not escape without parameter retuning, as a consequence of symmetry presence. In contrast, MPC2, RL and ND model successfully reached a target across all environments. In symmetric environments, where the robot starts equidistant from the available targets (Fig.~\ref{fig:benchmark}, first and second rows), MPC2 produced optimal trajectories, exhibiting an even split between the equivalent goals. The RL policy was also capable of reaching the targets; however, its behavior depended strongly on the order in which targets were presented in its observation vector. Regardless of the robot’s initial position and the efforts to avoid internal biases, the policy consistently favored one of the targets, indicating a reliance on input sequence rather than geometric proximity when the targets are approximately equidistant. Consequently, the RL trajectories were unevenly distributed and often suboptimal despite the symmetric configurations. Supported by the statistical analysis over RL policies (see Supplementary Materials Sec.~1.6), we conjecture that perfect symmetrical weight distribution in the neural network cannot be achieved, so the RL policy relies on these small imbalances to navigate symmetrical configurations. In contrast, the ND model exhibited balanced and unbiased decision-making, with approximately equal selection frequencies among the symmetric goals. Also, the ND model naturally decomposes multi-choice scenarios into a series of binary spatial-temporal decisions as the trajectory map shown in three symmetric target setting, replicating the decision-making behavior observed in animals and previously modeled using stochastic coupled spin models grounded in statistical physics \cite{sridhar2021geometry,gorbonos2024geometrical}.

\begin{figure}[htbp]
  \centering
  \setlength{\tabcolsep}{4pt}   
  \begin{tabular}{cccc}
    \begin{subfigure}[t]{0.22\textwidth}\centering
      \includegraphics[width=\linewidth]{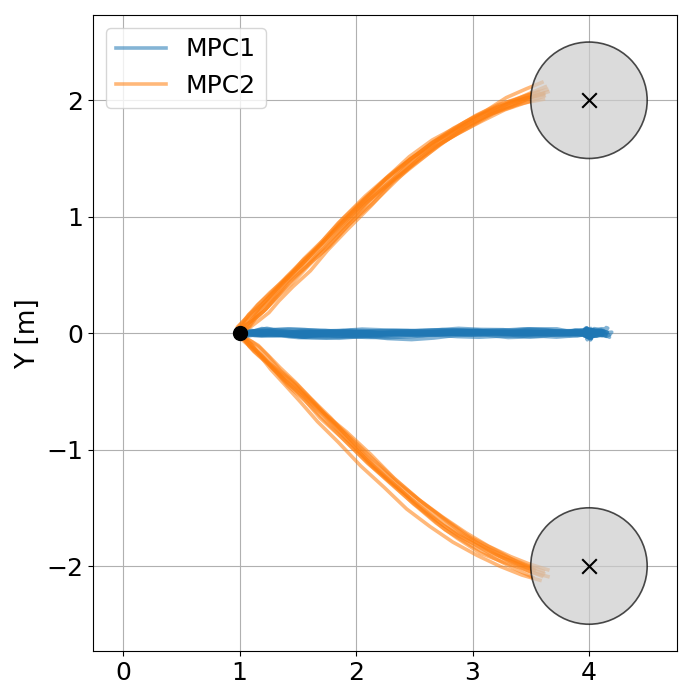}
      \label{fig:benchmarktraj1}
    \end{subfigure} &
    \begin{subfigure}[t]{0.22\textwidth}\centering
      \includegraphics[width=\linewidth]{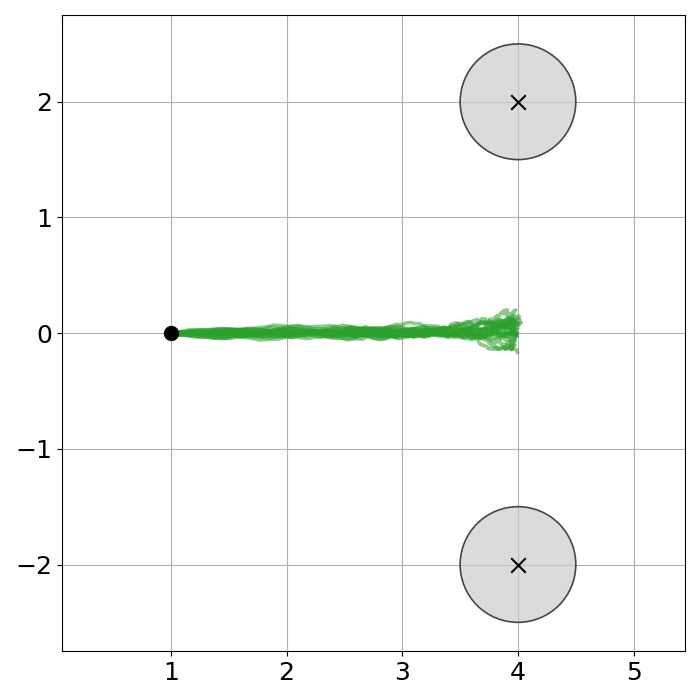}
      \label{fig:benchmarktraj2}
    \end{subfigure} &
    \begin{subfigure}[t]{0.22\textwidth}\centering
      \includegraphics[width=\linewidth]{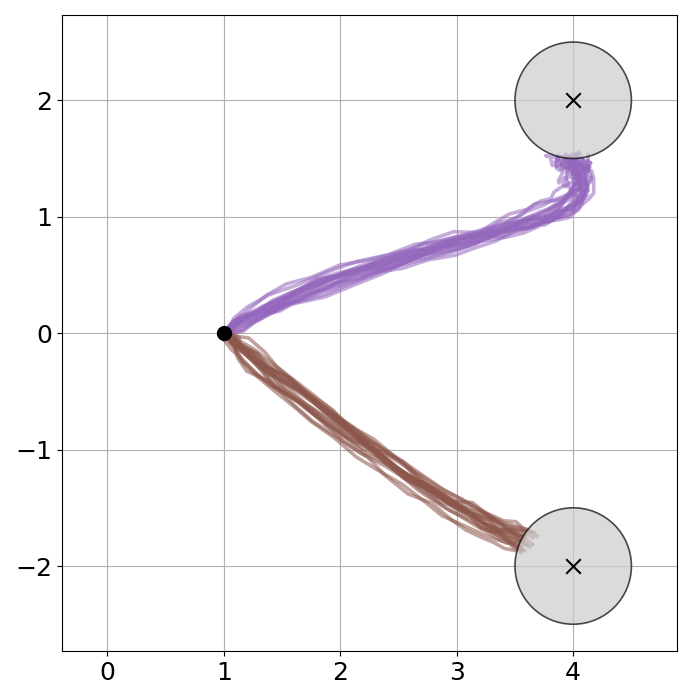}
      \label{fig:benchmarktraj3}
    \end{subfigure} &
    \begin{subfigure}[t]{0.22\textwidth}\centering
      \includegraphics[width=\linewidth]{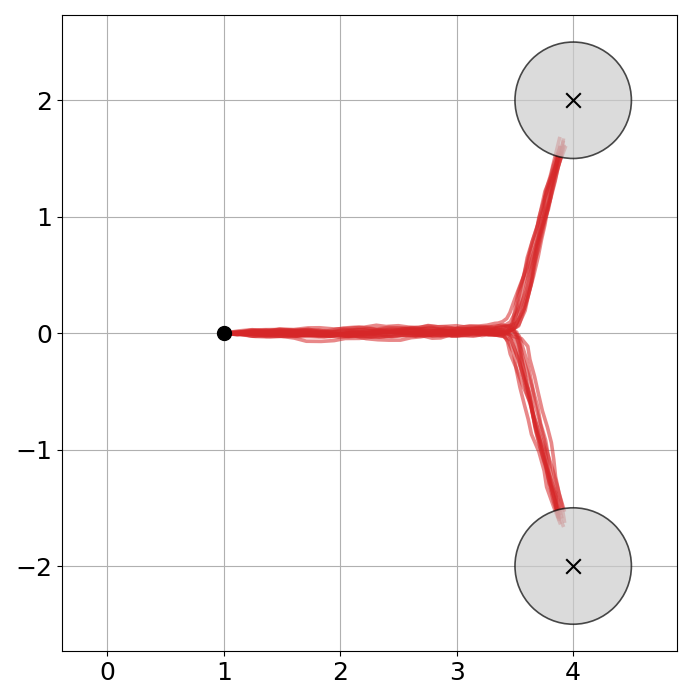}
      \label{fig:benchmarktraj4}
    \end{subfigure} \\
    \begin{subfigure}[t]{0.22\textwidth}\centering
      \includegraphics[width=\linewidth]{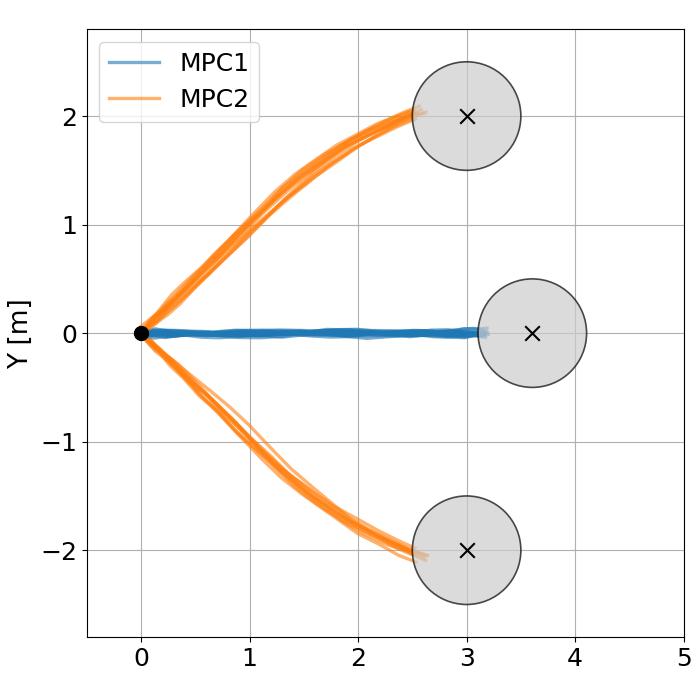}
      \label{fig:Benchmarktraj2_1}
    \end{subfigure} &
    \begin{subfigure}[t]{0.22\textwidth}\centering
      \includegraphics[width=\linewidth]{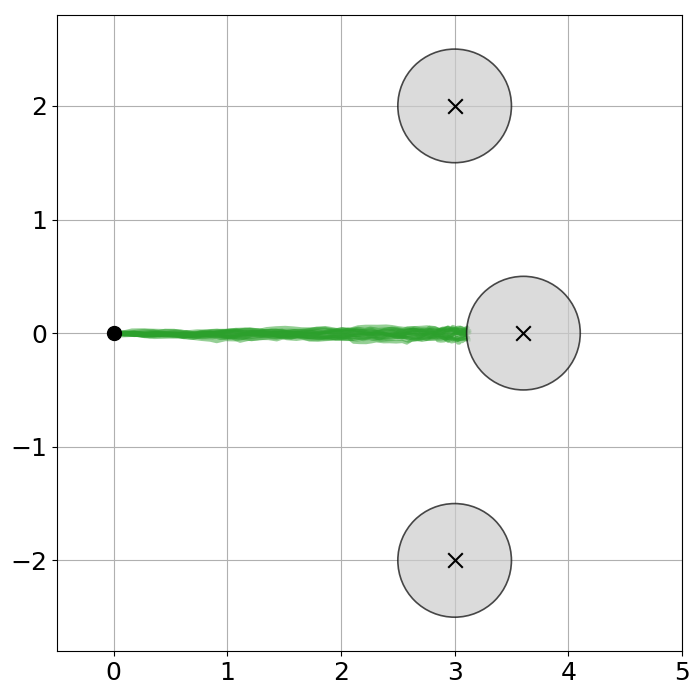}
      \label{fig:Benchmarktraj2_2}
    \end{subfigure} &
    \begin{subfigure}[t]{0.22\textwidth}\centering
      \includegraphics[width=\linewidth]{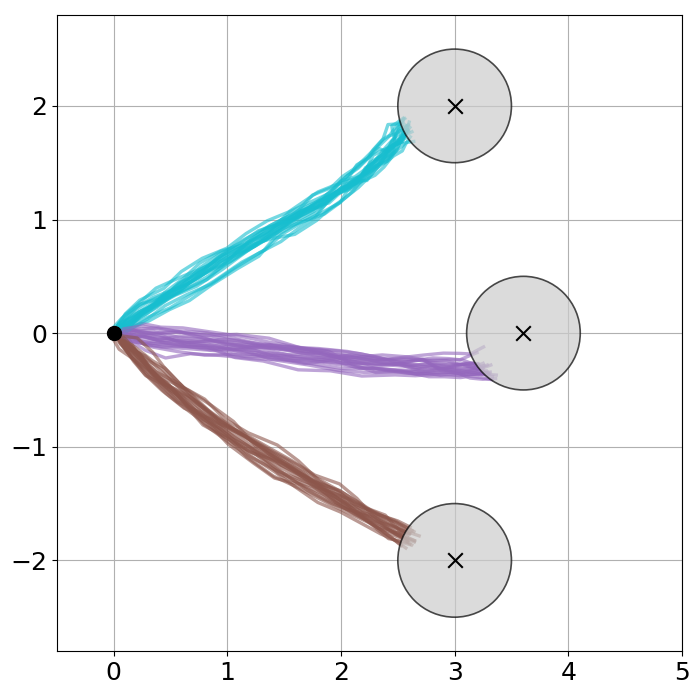}
      \label{fig:Benchmarktraj2_3}
    \end{subfigure} &
    \begin{subfigure}[t]{0.22\textwidth}\centering
      \includegraphics[width=\linewidth]{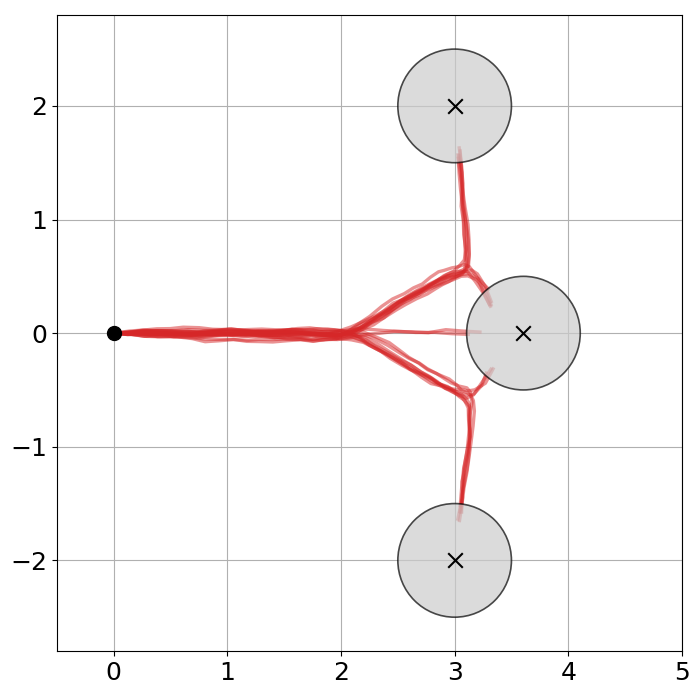}
      \label{fig:Benchmarktraj2_4}
    \end{subfigure} \\
    \begin{subfigure}[t]{0.22\textwidth}\centering
      \includegraphics[width=\linewidth]{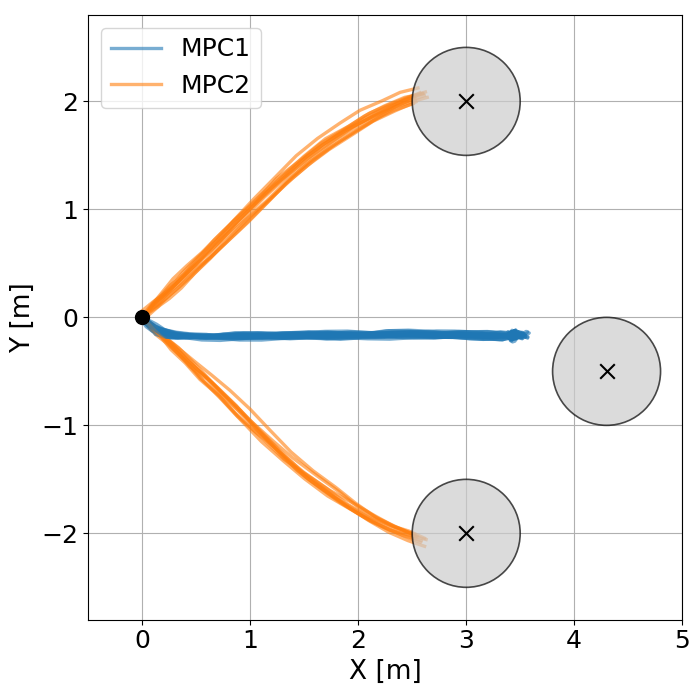}
      \subcaption*{MPC}
      \label{fig:Benchmarktraj3_1}
    \end{subfigure} &
    \begin{subfigure}[t]{0.22\textwidth}\centering
      \includegraphics[width=\linewidth]{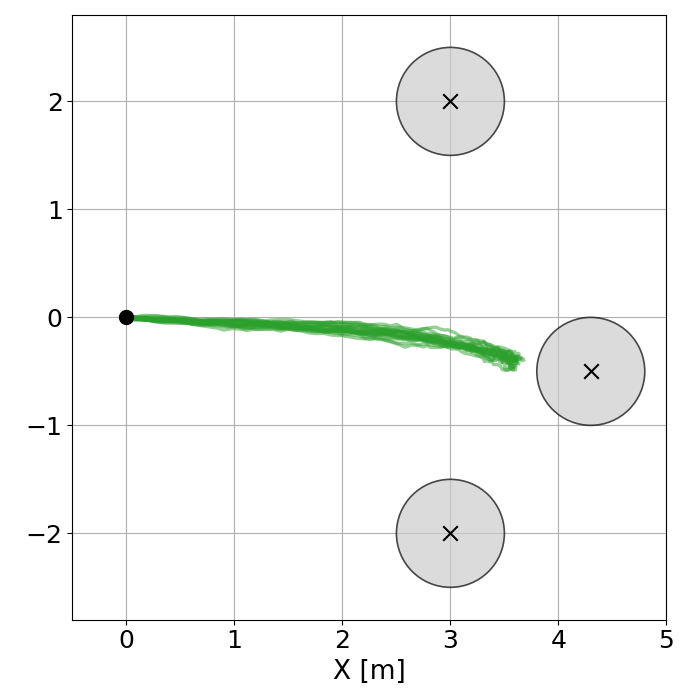}
     \subcaption*{PF}
      \label{fig:Benchmarktraj3_2}
    \end{subfigure} &
    \begin{subfigure}[t]{0.22\textwidth}\centering
      \includegraphics[width=\linewidth]{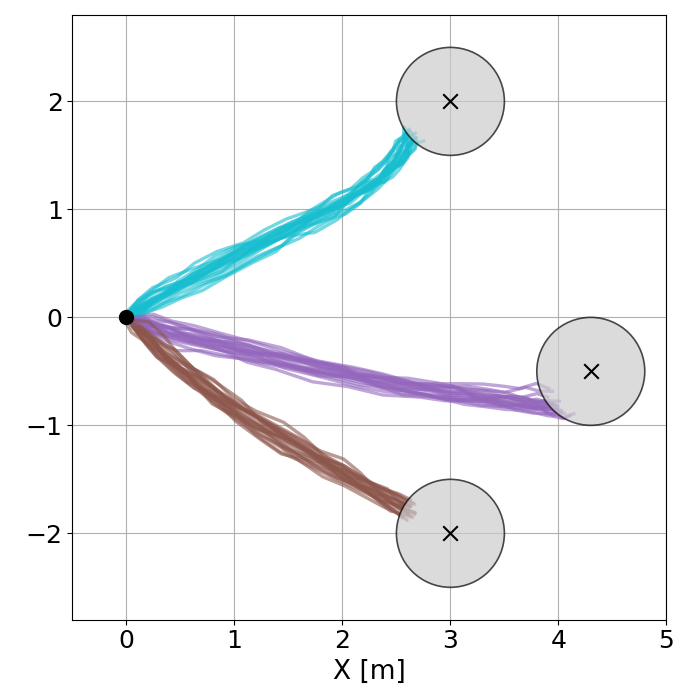}
      \subcaption*{RL}
      \label{fig:Benchmarktraj3_3}
    \end{subfigure} &
    \begin{subfigure}[t]{0.22\textwidth}\centering
      \includegraphics[width=\linewidth]{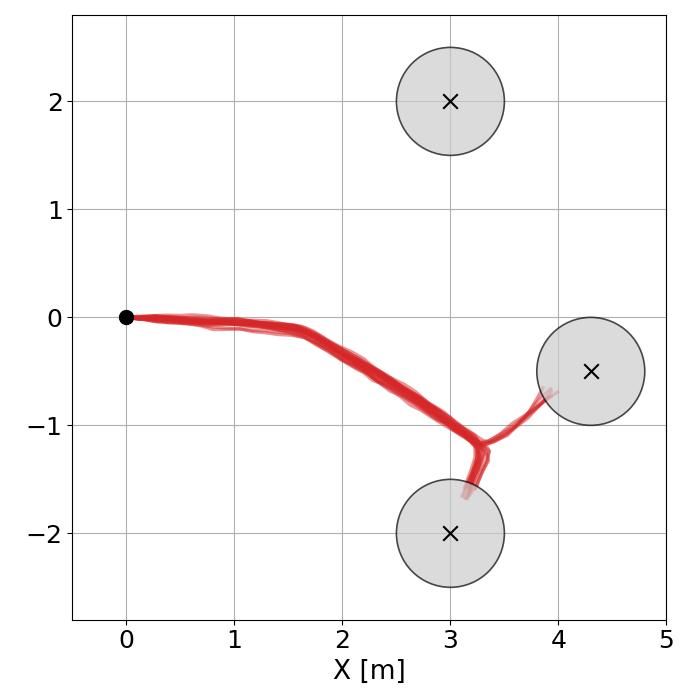}
      \subcaption*{ND}
      \label{fig:Benchmarktraj3_4}
    \end{subfigure}
  \end{tabular}
    \caption{\textbf{The geometry of decision-making for MPC, PF, RL and our method.} Each subplot shows $20$ simulated trajectories of the robot navigating to one of the targets (indicated by gray filled disks of radius $\epsilon = 0.5$ m), starting from the origin. The four planners are tested across three environments to evaluate their ability to break symmetry and exploit environmental geometry, from top to bottom: symmetric two-goal, symmetric three-goal, and asymmetric three-goal. For MPC (left column), two formulations are compared: MPC1 (weighted-sum, blue) and MPC2 (soft-min, orange). For RL, the same policy is used but with different target orderings provided during simulation, where each color corresponds to a distinct goal sequence as observations. For ND, the saturation  bound parameter is $a=2$ and the coupling gain is $\alpha = 6$.  }
    \label{fig:benchmark}
\end{figure}

We also evaluated an asymmetric environment, where the robot faced three targets: two positioned symmetrically and a third placed closer to the lower target group (Fig.~\ref{fig:benchmark}, last row). This configuration introduces an intentional asymmetry that created a natural bias toward the lower region and the objective is to assess whether each planner can exploit this asymmetry when making a decision. A desirable behavior is one in which the planner demonstrates a consistent preference for the lower target region, reflecting sensitivity to the environment’s geometric structure. We found that MPC2 has a near-even split between two symmetric ones since their distance to the initial position is  closer. We see the same phenomenon in the RL policy. In contrast, the ND model consistently selects the lower target group, demonstrating a consistent preference aligned with the spatial asymmetry. As shown in Fig.~\ref{fig:benchmark}, all trajectories converge toward the lower two targets indicating that the model effectively exploits the spatial asymmetry rather than relying on noises or randomization. We collected the mean and standard deviation of per-step computation times for all planners in two-goal and three-goal setting (see Table S\ref{tab:comp_time1}-\ref{tab:comp_time2}).  The ND model exhibits extremely low and consistent computational costs---on the order of $0.1$ms per step---which is 3-4 orders of magnitude faster than MPC planners and $5\times$ faster than the RL policy.

Beyond its ability to break symmetry and exploit geometric asymmetries for decision-making, Fig.~\ref{fig:benchmark} also shows that the ND model does not commit to a target immediately, unlike MPC2 and the RL policy. Instead, it delays the decision until the trajectory reaches the bifurcation point, where sufficient evidence has accumulated to support a reliable choice. This behavior is especially illuminating when transitioning from GPS-based to vision-based inputs (as will be studied in Secs. \ref{Sec2_2}-\ref{Sec2_4}). When the robot is far from the targets, visual cues are weak---targets appear small, blurry, and provide limited information about their relative importance. As the robot moves closer, the visual evidence becomes richer and more reliable, naturally guiding the decision process. This makes the ND model particularly well suited for real-time autonomy, as it remains reactive to the environments while maintaining lightweight computation.

\subsection{Vision-based decision-making on the move}
\label{Sec2_2}
In the simulation studies in the previous section, the ND model used global coordinates of targets to decide on a movement direction. Next, we introduce the role of perception and evaluate the proposed neuromorphic control framework using analytically generated visual inputs in 2D plane. Instead of rendering raw images, for this study the visual input matrix is computed directly from the robot-target geometry in real-time. The robot is equipped with a forward-facing pinhole camera with a finite FOV and targets are modeled as spheres for simplicity, although the method extends to arbitrary shapes. A spherical target within the FOV subtends an apparent half-angle
\[
\gamma_t = \arctan\!\left(\frac{r_t}{\|\mathbf{p}_t - \mathbf{p}_c\|}\right)
\approx \frac{r_t}{\|\mathbf{p}_t - \mathbf{p}_c\|},
\]
where $r_t$ and $\mathbf{p}_t$ are the target’s radius and center, and $\mathbf{p}_c$ is the camera center. 
In the pinhole model, each pixel $(i,j)$ corresponds to a unit ray $\hat{\mathbf o}_{ij}$ that originates from the camera center. A target  is ``visible'' in pixel $(i,j)$ if the angular deviation between $\hat{\mathbf o}_{ij}$ and the direction $\hat{\mathbf u}_t = (\mathbf{p}_t - \mathbf{p}_c)/\|\mathbf{p}_t - \mathbf{p}_c\|$ satisfies
\[
\arccos\!\big(\hat{\mathbf o}_{ij}\!\cdot\!\hat{\mathbf u}_t\big) \le \gamma_t.
\]
Using this geometric test, we populate $\mathcal{U}$ (see Section \ref{ref:model_form} in Methods) from poses, without acquiring real image streams. The resulting sensor model is idealized and highlights how the decision-making dynamics operate with high-dimensional inputs.
\begin{figure}
    \centering
    \includegraphics[width=1\linewidth]{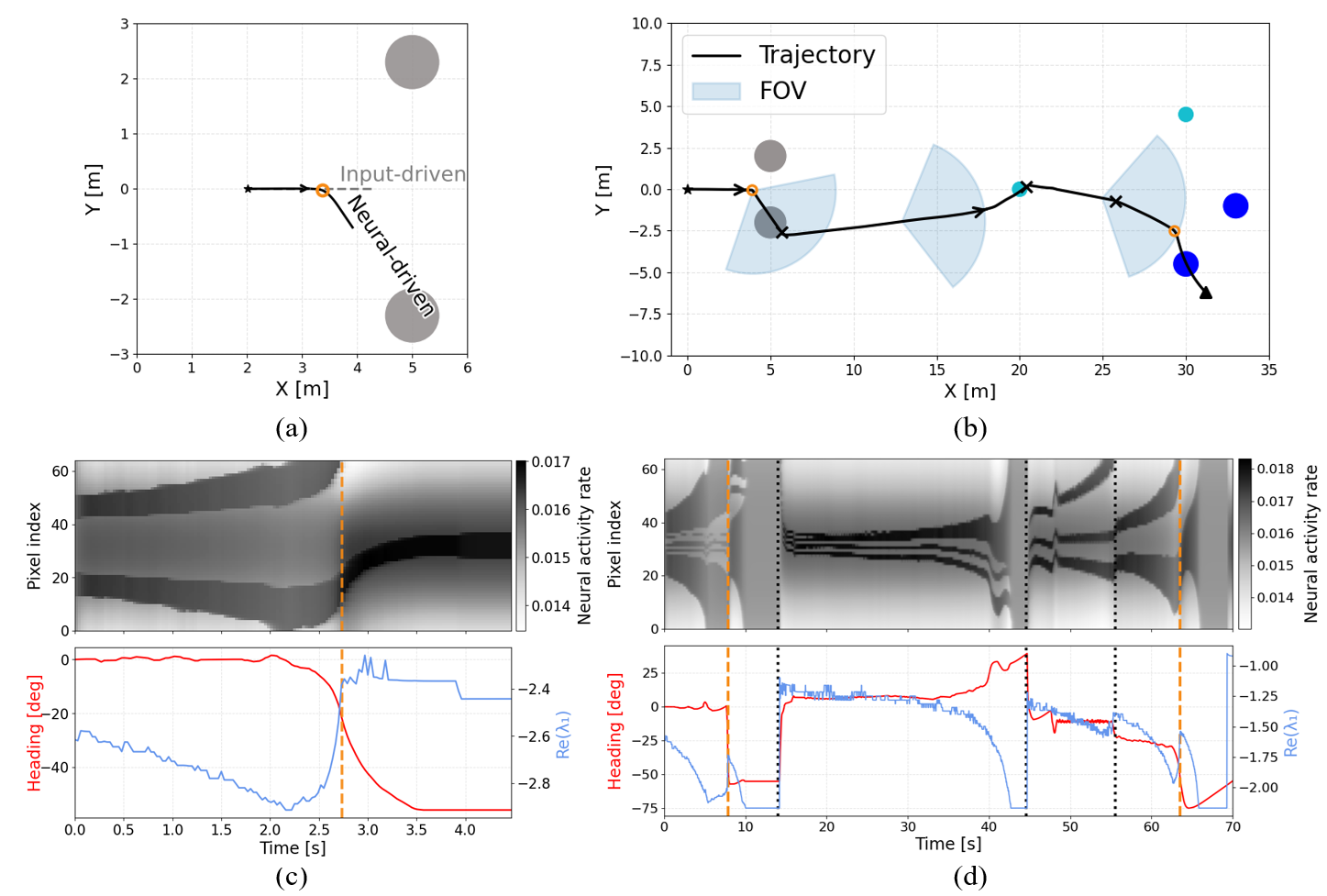}
    \caption{\textbf{Autonomous navigation with simulated visual inputs.} The robot has a FOV of $110^{\circ}$ and  neuron population $k = 64$ ($a = \frac{8}{k}$ and $\alpha = 4.2$).
    \textbf{a}  Trajectories in the symmetric setting. The red curve shows the path generated by the neural-dynamics controller, while the gray dashed curve corresponds to the trajectory driven purely by the instantaneous visual input (the $\mathcal{U}$ matrix). The circular  marker on the trajectory denotes the moment at which the neural dynamics break symmetry and commit to one of the two identical targets.
    \textbf{b}   Full goal-directed trajectory; shaded blue wedges indicate the FOV; circular targets are color-coded by radius. The orange circular  marker on the trajectory denotes bifurcation-driven  decisions while the black cross markers denote vision-driven decisions.
    \textbf{c} Top: Neural activity heatmap showing how the population response evolves over time across pixel indices. Bottom: Corresponding heading angle (relative to the $+x$ axis) and the real part of the dominant Jacobian eigenvalue $\lambda_1$. The vertical dashed line marks the same symmetry-breaking moment indicated by the circular  marker in (a), where activity concentrates onto one target band, the heading rapidly changes, and $\operatorname{Re}(\lambda_1)$ spikes. All panels share a common time axis and correspond to the neural-driven trajectory in (a).
    \textbf{d} Same panel as \textbf{c} corresponding to the trajectory in \textbf{b}.}
    \label{fig:sim_py}
\end{figure}

In the symmetric configuration shown in Fig.~\ref{fig:sim_py}a, the robot is initialized near a bifurcation point. The timescale of the neural dynamics is set to be three times faster than that of the velocity update, ensuring convergence of neural activity rates and revealing the coupling between neural and sensory processes. At this decision-making point, the neural dynamics evolve under symmetric visual evidence until the system commits to one of the alternatives. For comparison, we also simulate the trajectory driven purely by raw visual inputs, i.e., $\mathbf{v} = v_0 \mathcal{P}\mathbf{u}$ (gray dashed line in Fig.~\ref{fig:sim_py}a). While this input-driven controller simply averages between the two identical options, the neural dynamics experience a qualitative reorganization at around $2.7$s (Fig.~\ref{fig:sim_py}c, top), transitioning sharply from indecision to choice. This transition reflects the intrinsic coupling between neural state and perceived input: changes in neural activity induce abrupt shifts in the robot’s heading, which in turn reshape the incoming visual stimuli, as seen between $2$–$3\text{s}$ in Fig.~\ref{fig:sim_py}b. The parsimony of the neural dynamics allow to study the decision-making process of the robot by studying how the largest Jacobian eigenvalue $\lambda_1$ of the neural dynamics evolves with time. Ultra-sensitivity happens when $\lambda_1$ exhibits a spike (see Sec~\ref{sec:timescale_bif} in Methods for details). Accordingly, the heading direction exhibits a change in ultra-sensitivity regions, confirming the causality between visual stimuli, neural activity and robot movement.

Building on the bifurcation analysis, we synchronize the robot’s translational dynamics with the timescale of neural decision-making. As shown in Fig.~\ref{fig:sim_py}b, the robot initially faces a near symmetric configuration in which two gray targets have equal radii. In this setting, the neural dynamics exhibit rapid symmetry breaking: although both options are equally weighted, the system spontaneously commits to one and maintains a stable trajectory toward it rather than averaging or remaining indecisive.
In the asymmetric scenario, where blue targets have unequal radii, the robot preferentially selects the larger target. This bias arises because the larger target subtends a wider visual angle, activates more neurons through the visual input matrix, and consequently generates a stronger drive within the neural decision dynamics. In Fig.~\ref{fig:sim_py}b, orange circular markers denote positions where the neural dynamics trigger the decision (corresponding to a spike in $\lambda_1$), whereas black cross markers indicate visually induced decisions, such as when a target leaves the FOV or when the robot passes directly through a target.
As shown in Fig.~\ref{fig:sim_py} (bottom row), $\lambda_1$ exhibits a spike at each decision point, either driven by visual input or neural activity.

\subsection{Vision-based neuromorphic navigation and tracking}
\label{Sec2_3}
Having analyzed the decision-making mechanism of the neuromorphic framework in isolation, this section studies the framework as a whole for navigation and tracking in a photorealistic simulation environment.We use AirGen \cite{vemprala2023gridplatformgeneralrobot} as our flight simulation environment. AirGen provides environment models, virtual sensor modules, airframe dynamics, a physics engine, a rendering interface, and an API for integrating custom control modules. We used a quadrotor in all the simulations. The results of a representative navigation trial are shown in Fig.~\ref{fig:Grid_sim1}, with the full visualization provided in Video S2. The quadrotor operates in an environment populated with six static targets, indicated by green polyhedra distributed throughout the workspace. We employ a $36\times64$ neural population that encodes visual input from an onboard RGB camera with a $120^{\circ}$ FOV and $144\times256$ pixel resolution. The complete trajectory reconstructed from global state logs demonstrates a smooth, goal-directed path through the environment, as illustrated in the top-down (Fig.~\ref{fig:grid1_2}) and 3D visualizations (Fig.~\ref{fig:grid1_1}). 

To further elucidate the neural mechanisms underlying the observed behavior, Fig.~\ref{fig:grid1_3} depicts the evolution of visual and neural representations at six key moments during the same trial.
The first-person perspective snapshots (Fig.~\ref{fig:grid1_3}, top row) correspond to distinct phases such as approaching target, transition, and decision-making between near-symmetric goals. The second row shows the pixel-wise target representation map that encodes the likelihood of target presence from the latest image input. Meanwhile, the neural activity rates represent integrated dynamics that evolve continuously as the system processes successive inputs. Thresholded neuronal activity (Fig.~\ref{fig:grid1_3}, last row) further highlight the subset of neurons contributing to motion generation and provided a denoised representation of neuronal activities.

\begin{figure}[htbp]
     \centering
    \begin{subfigure}[b]{0.48\textwidth}
    \centering
         \includegraphics[width=\textwidth]{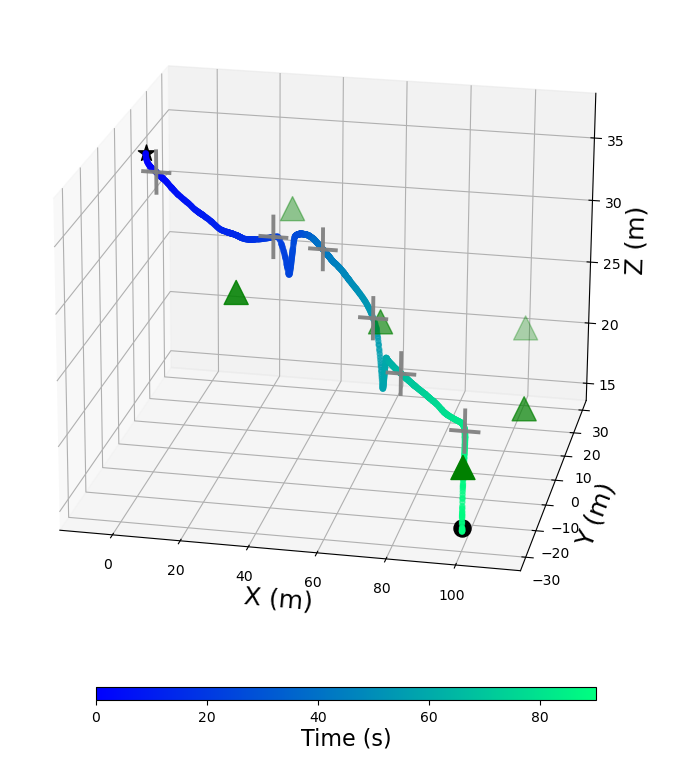}
         \caption{}
         \label{fig:grid1_1}
    \end{subfigure}
    \begin{subfigure}[b]{0.48\textwidth}
    \centering
         \includegraphics[width=\textwidth]{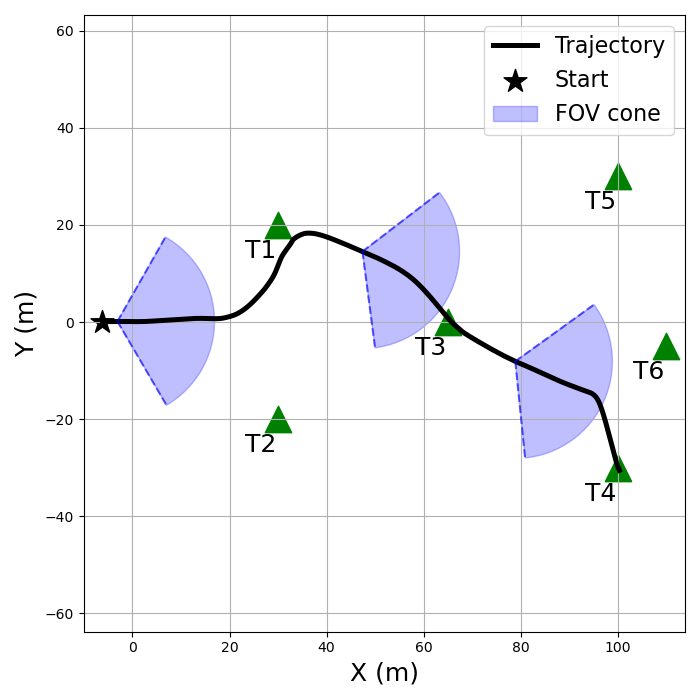}
         \caption{}
         \label{fig:grid1_2}
    \end{subfigure}
    \begin{subfigure}[b]{\textwidth}
    \centering
         \includegraphics[width=0.95\textwidth]{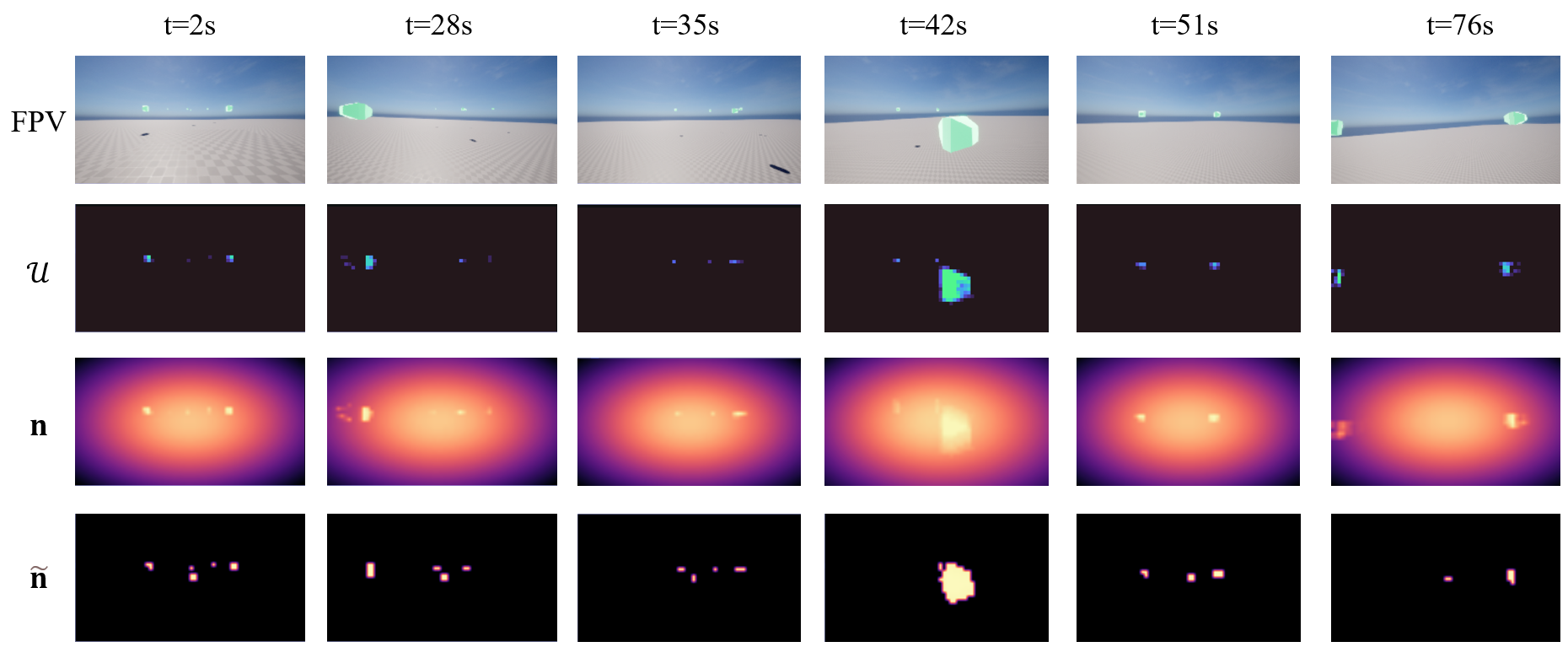}
         \caption{}
         \label{fig:grid1_3}
    \end{subfigure}
     \caption{
     \textbf{Vision-based navigation in a photorealistic environment.} 
     The quadrotor has a FOV of $120^{\circ}$ and  neural population  $k=36\times64$ ($a=\frac{5}{k}$ and $\alpha = 3$). \textbf{a} 3D trajectory. The star marks the start and the circle denotes the end of the path. Gray cross marks indicate six key positions whose snapshots are shown in (c).
     \textbf{b} Top-down projection of the trajectory in (a), with the FOV cones overlaid.
    \textbf{c} First-person perspective snapshots (top row) and their corresponding intermediate representations at six key moments: (1) start, (2) approaching target T1, (3) transitioning from T1 to T3, (4) approaching T3, (5) choosing between T4 and T6, and (6) approaching T4. The second row shows the pixel-wise target evidence map $\mathcal{U}$. The third row visualizes the neural activity rates representing the population activity distribution across all neurons. The bottom row shows the neural activity after thresholding, highlighting neurons that drive motion (see Video S2 for more visualizations).
    }
    
     \label{fig:Grid_sim1}
\end{figure}

We further evaluate the proposed framework in a multi-target tracking scenario, where a quadrotor is tasked with monitoring multiple flying quadrotors whose trajectories are unknown. The following assumptions are made: (1) each target exhibits smooth motion with bounded velocity and acceleration changes, and (2) the quadrotor aims to preserve a stable relative position with respect to the moving targets. Target tracking has been extensively studied in the UAV literature e.g., \cite{Thomas17ral, Yadav21cdc, chung2018survey, li2016multi, li2024learning, sun2024moving}. Most existing works focus on single-object tracking, where one or multiple quadrotors are controlled to keep a single target within their FOV \cite{Thomas17ral, Yadav21cdc}. Several studies have addressed multi-object tracking by assigning unique identifiers to each target and manually defining tracking assignments, e.g., \cite{chung2018survey, li2016multi, li2024learning, sun2024moving}. In contrast, our approach does not rely on explicit target IDs or manual assignment. Instead, it autonomously decides which targets to prioritize based solely on their visual observability.

To obtain visual representations, we employ an optical-flow camera that captures the motion field $\mathbf{u}$, demonstrating the adaptability of our framework to various sensing modalities. The camera has a $120^{\circ}$ FOV and a resolution of $260 \times 144$, which is encoded by a neural population of $65 \times 36$ units. The complete 3D trajectories of the controlled quadrotor and the targets, reconstructed from global state logs, are shown in Fig.~\ref{fig:grid2_1}  (see also Video S3). During this scenario, there are three distinct phases. First, all targets maintain formation. Next, one target departs from the formation while the other two remain together. Finally, the remaining two targets eventually separate. Across these phases, the proposed decision-making mechanism exhibits three characteristic behaviors. First, the quadrotor maintains formation with the targets when feasible. Next, when one of the targets splits the formation, the quadrotor preferentially selects and tracks the sub-group subtending a larger visual angle in the image plane, continuously adjusting its heading to keep the chosen targets centered within the camera view. Because the targets occupy a small fraction of the visual field and often move with relatively stable velocities, the resulting optical flow measurements can be noisy or even fail to yield clear target evidence. Nevertheless, our framework successfully maintains robust tracking and decision-making performance. This resilience arises from the inherent memory of the neural dynamics, which retain traces of past sensory inputs. For example, in Fig.~\ref{fig:grid2_3} at $t=17$s, the neural activity $\mathbf{n}$ continues to represent the right-hand target quadrotor even when the current visual input $\mathcal{U}$ temporarily loses track of it, demonstrating the system’s ability to maintain perceptual continuity under sensory uncertainty.

\begin{figure}[htbp]
     \centering
    \begin{subfigure}[b]{0.50\textwidth}
    \centering
         \includegraphics[width=\textwidth]{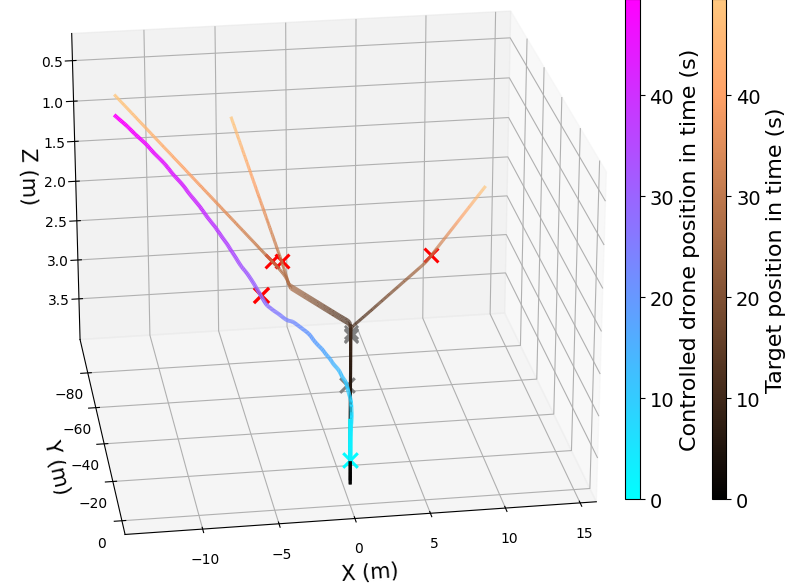}
         \caption{}
         \label{fig:grid2_1}
    \end{subfigure}
    \begin{subfigure}[b]{0.40\textwidth}
    \centering
         \includegraphics[width=\textwidth]{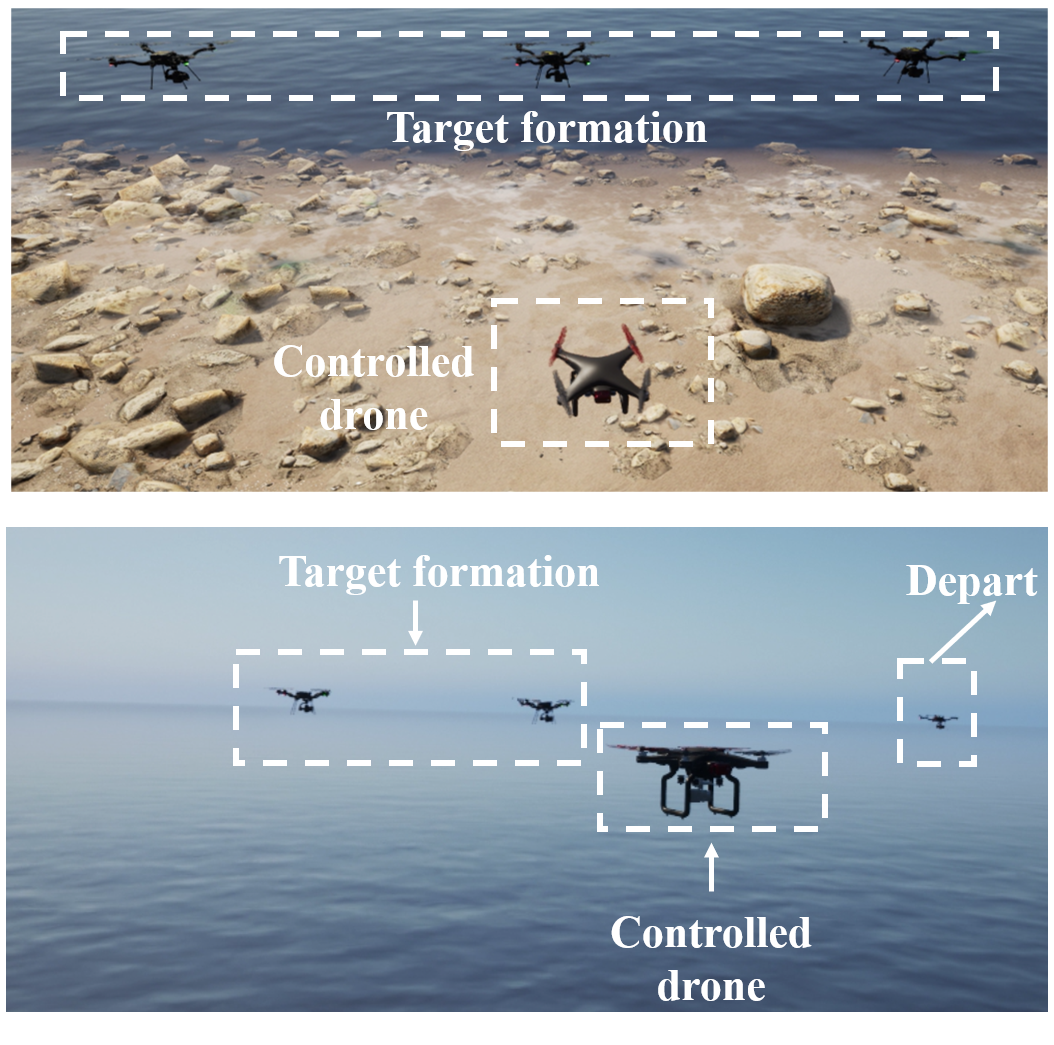}
         \caption{}
         \label{fig:grid2_2}
    \end{subfigure}
    \begin{subfigure}[b]{\textwidth}
    \centering
         \includegraphics[width=0.94\textwidth]{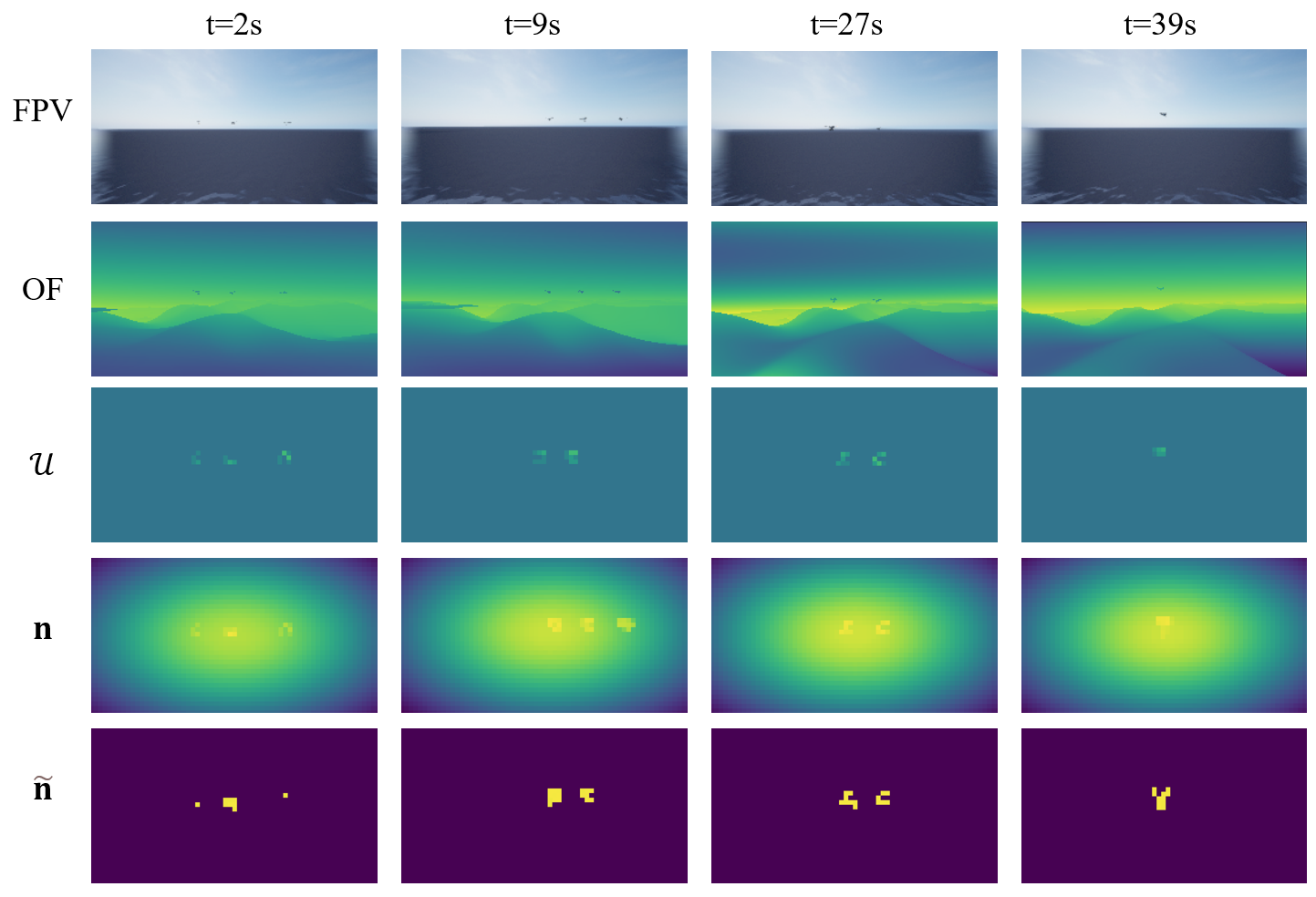}
         \caption{}
         \label{fig:grid2_3}
    \end{subfigure}
     \caption{
     \textbf{Vision-based multi-target tracking in a photorealistic environment.}
    The controlled quadrotor has a $120^{\circ}$ FOV and a population of $k = 36\times64$ neurons ($a=\frac{6}{k}$ and $\alpha = 6$).
    \textbf{(a)} 3D trajectories of the controlled quadrotor and three target quadrotors, with color gradients indicating temporal evolution. Gray and red cross-markers highlight the quadrotor positions at two moments, $t=9$s and $t=27$s, respectively. The corresponding first-person snapshots and intermediate representations at these times are shown in the second and third columns of panel (c).
    \textbf{(b)} “Fly-with-me” view at two moments. Top: At the start, the controlled quadrotor (white dashed box) observes all three targets flying in formation (yellow dashed box) within its FOV. Bottom: the rightmost target (black dashed box) departs from the formation while the controlled quadrotor continues tracking the group.
    \textbf{(c)} First-person perspective (FPV) snapshots and intermediate neural representations at four key moments: (1) start, (2) before one quadrotor leaves formation, (3) while following the larger formation of two targets, and (4) after switching to track a single target. The second row shows residual optical flow (OF) magnitude, and the third row shows the target evidence map $\mathcal{U}$. The fourth row visualizes the neural activity distribution while the bottom row highlights thresholded neural activations that directly drive motion decisions (see Video S3 for full visualizations).
    }
    
     \label{fig:Grid_sim2}
\end{figure}

\subsection{Autonomous decision-making in the physical world}
\label{Sec2_4}
Finally, we deployed the proposed neuromorphic framework in a real, custom-built quadrotor and conducted an statistical analysis of a navigation task involving two symmetrical targets (see Video S1).  The quadrotor, displayed in Fig.~\ref{fig:pipeline}a, ran all the computations onboard and was equipped with a camera that acquired images at 15 Hz. The parameters of the neural dynamics were adjusted to ensure that the spatial coordinates of the bifurcation point are located in the forward direction of the quadrotor in the robot frame. By doing that, we ensured statistical analysis of both the indecision region and the decision-making point. The quadrotor only relied on its locally available information during all the experiments, but a motion capture system was used to retrieve accurate measurements of the quadrotor position to ensure unbiased statistical analysis. As detailed in Fig. \ref{fig:real_experiments}, the quadrotor responded to visual stimuli in the same way observed in the simulated experiments, moving forward in the indecision region until the strength of target evidence induces a decision in the neural dynamics. Once the decision is taken, the neural dynamics yield a velocity command that orients the quadrotor to the chosen target and completes the navigation task. 

\begin{figure}
    \centering
    \includegraphics[width=0.95\linewidth]{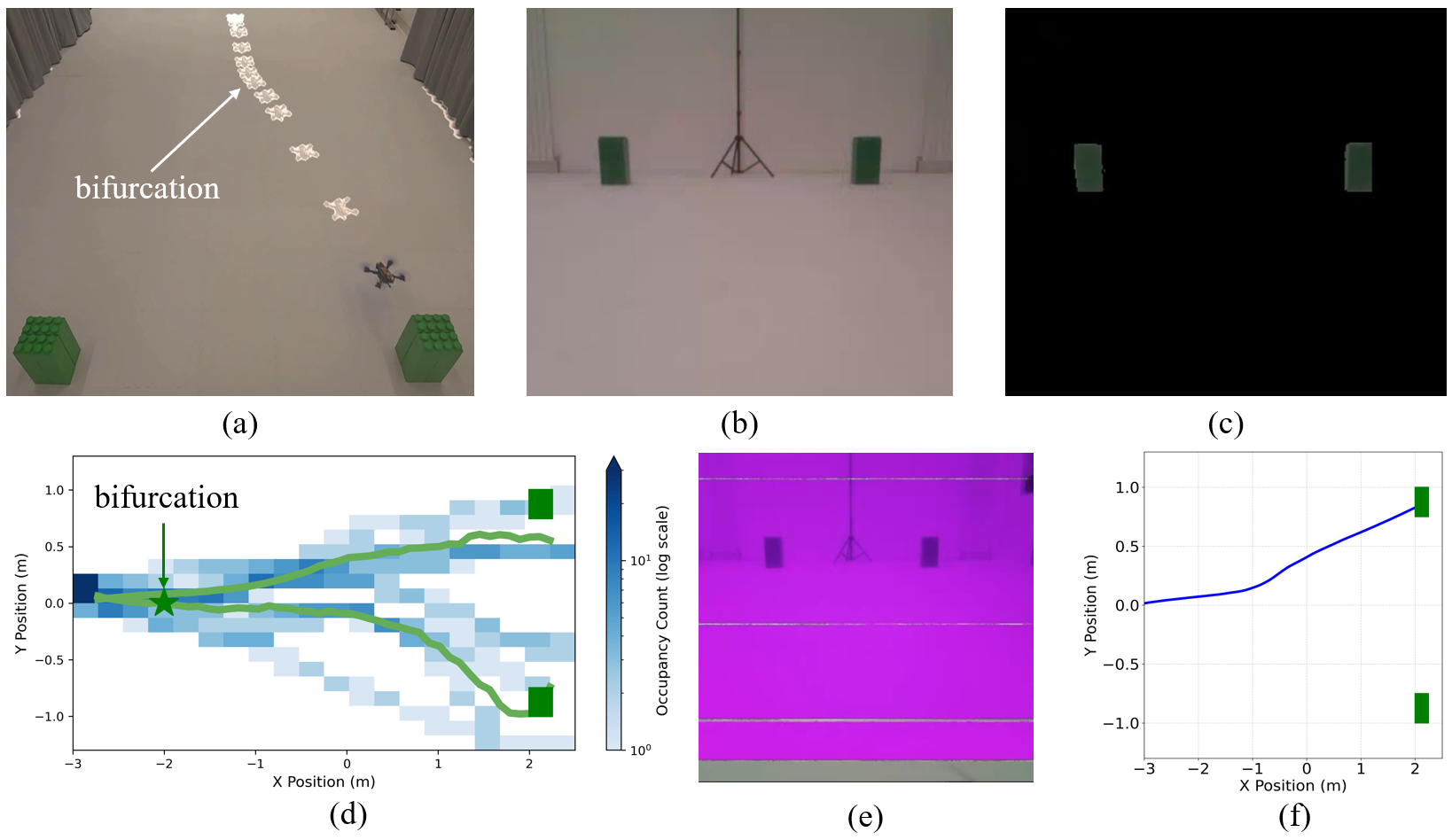}
    \caption{\textbf{Experimental results with a quadrotor.} \textbf{a} A custom quadrotor starts in a symmetrical configuration with respect to two targets, both of them within its FOV. The quadrotor moves forward steered by the indecision region of the neural dynamics until the bifurcation point is reached. Then, the ultra sensitivity region in the neural dynamics ensures that a decision is made, in this case the left target (with respect to the quadrotor body frame). Once a target is chosen, the quadrotor faces the target and progress towards it. \textbf{b-c} The quadrotor acquires RGB images at 15 Hz and uses a Hue-Saturation-Value color mask to compute the $\mathcal{U}$ matrix. 
    \textbf{d} A statistical analysis conducted shows that the quadrotor, presented with the same situation 10 times, bifurcates at (-2, 0). After that, the smoothed averages (window of 5 steps, in green) over the trajectories that turn left and right evolve according to the theoretical analysis of the neural dynamics, with the quadrotor facing the chosen target (green squares) and moving towards it. Statistical consistency is achieved despite noise in the output from perturbations in battery voltage and mismatches between commanded velocity and actual velocity. 
    \textbf{e-f} The proposed neuromorphic framework also exhibits high levels of disturbance rejection against noise in the input. During the experiment collection phase, the cable connecting the camera and the Raspberry Pi was damaged, leading to unpredictable perturbations affecting the entire image, as pictured in panel \textbf{e}. Nonetheless, the quadrotor was able to complete the task successfully, displaying the same decision-making behavior as the experiments with nominal input conditions, pictured in \textbf{f}. 
    }
    \label{fig:real_experiments}
\end{figure}

The experiments were conducted in reproducible conditions of initial position and orientation of the quadrotor and the targets. However, the collection of experiments included runs with different battery levels, yielding to unpredictable disturbances that affected the low-level module that translates velocity commands into actionable thrust commands. Also, the placement of the battery beneath the quadrotor frame was based on adhesive strips, meaning that the center of gravity of the platform was subject to constant disturbance that affected the stability of the low-level controller and its internal Extended Kalman Filter estimator. In spite of all these sources of noise, a statistical analysis of the trajectories of the quadrotor across experiments demonstrates that the proposed neuromorphic controller followed the same behavior exhibited by other animals subject to the same decision-making geometry \cite{forstmann2016sequential, ratcliff2016diffusion, roxin2008neurobiological}, with the bifurcation happening around $(-2.0, 0.0)$m and the indecision region being concentrated in the region between $(-2.8, 0.0)$m and $(-2.0, 0.0)$m. 

Another relevant factor to consider in these experiments is the limited FOV of the camera. With a FOV of $60^{\circ}$, output perturbations led to eventual disappearances of the target from the image, inducing a critical input perturbation that our neuromorphic framework can handle naturally. The neural dynamics include a memory term that rejects high-frequency disturbances in the input. As a consequence, the quadrotor, with the appropriate tuning of the neural dynamics, can continue its navigation for sufficiently long enough time to recover from the eventual loss of visual stimuli. Furthermore, during the collection of experiments, the image stream was corrupted, leading to RGB inputs such as the one displayed in Fig.~\ref{fig:real_experiments}e. Those experiments were not included in the statistical analysis but, for completeness, we studied the trajectories achieved by the quadrotor and, interestingly, it performed similarly than with the appropriate hardware conditions. The quadrotor still exhibited an indecision region around the same spatial bifurcation point, beyond which it committed to one of the two targets, and oriented and moved toward it. In this case, the corruption in the color space of the image was symmetrical in the same spatial orientation of the symmetry of the targets, so it did not induce any asymmetry that could help break indecision.

\section{Discussion}\label{sec12}
In this work, we introduced a parsimonious neuromorphic control framework that effectively bridges the architectural divide between proximal and distal navigation. By translating high-dimensional visual inputs directly into egocentric motion commands, our system retains the computational efficiency and reactivity of proximal, sensor-based approaches. At the same time, unlike traditional reactive methods that struggle with ambiguity, our bio-inspired system emulate the decisive capabilities of distal planners, successfully overcoming symmetry-induced indecision without relying on internal world models or explicit maps. This synthesis offers a reliable path to autonomy that eliminates the need for heavy planning stages or extensive neural network training, achieving decisive behaviors purely through sensorimotor dynamics.

The power of our framework lies in its ability to emulate the dynamic cognitive computations found in animals. The controller's decision-making process emerges from the collective dynamics of neural populations, which are governed by principles that are analogous to ring attractors and Nonlinear Opinion Dynamics. As we demonstrate through analysis, simulations, and physical experiments, when a robot encounters symmetrical targets, the system naturally enters an indecision region. As the robot moves forward, it reaches a spatial ``bifurcation point'' where the neural dynamics become ultrasensitive to differences between targets, biasing the decision towards options that subtend larger areas of the robot's visual field. At this point, asymmetries are amplified, forcing the system to commit decisively to a single target. This mechanism differs fundamentally from the explicit cost-function optimization of MPC or the learned biases of RL, enabling our method to robustly and repeatedly achieve unbiased symmetry breaking. Meanwhile, the computational load of the framework is dominated by two components: visual processing (target evidence extraction) and neural population. The former reflects a trade-off between disturbance rejection of real-world visual conditions and the fidelity of the extracted cues, while the latter scales primarily with the camera’s FOV, as wider FOVs require more neurons to represent the visual space. Meanwhile, we show that a neural population of size $64 \times 32$---sufficient to represent a FOV smaller than $180^{\circ}$---achieves a runtime that remains on the same order of magnitude as a reinforcement learning policy network with an 11-dimensional input, two 256-unit hidden layers, and a 2-dimensional output.

A compelling feature observed across our experiments is the system's ability to maintain consistent performance despite sensory imperfections and hardware instabilities. In multi-target tracking scenarios, where weak optical flow cues and small visual angles often lead to noisy detection, the system effectively maintained tracking and decision-making capabilities without explicit filtering modules. This behavior suggests that the temporal integration inherent to the neural dynamics provides a natural resilience against signal dropouts and high-frequency noise, a property that can be further explored in swarm coordination and flocking behaviors \cite{mezey2025purely, vasarhelyi2018optimized}. Similarly, our physical quadrotor experiments yielded consistent bifurcation behaviors even when subjected to significant disturbances, including fluctuating battery levels, shifting centers of gravity, and severe image stream corruption. While these results provide promising evidence of the controller's intrinsic tolerance to real-world noise, a systematic quantification of these stability margins, subjecting the system to controlled, high-magnitude adversarial perturbations, is required to characterize robustness. In this paper, we intentionally use simple visual cues to distinguish targets, drawing inspiration from how animals detect predators or food sources \cite{mauss2020optic, hemmi2005predator, teichroeb2014sensory}, but richer perception models can be integrated to handle more complex navigation tasks.

Although our results are promising, this framework raises several questions for future research. We only demonstrated navigation with up to six targets, so the scalability of the neural dynamics in scenes with dozens of potential targets or distractions remains to be quantified. In theory, our framework favorably scales with the total number of targets in the scene because the neural population only responds to stimuli within the camera's FOV. However, in practice, targets must be distinguishable in order to make an actual decision. Our experiments, which ran at 15 Hz, hinted robust navigation. Exploring the performance of this controller in high-speed, agile flight, perhaps by pairing it with event-based neuromorphic cameras for stronger background-foreground separation would be a compelling next step, as well as conducting a systematic assessment of robustness as a property emerging from neural dynamics using rigorous mathematical analysis of the framework and simulation studies. 

Ultimately, the proposed neuromorphic framework provides physical validation of computational neuroscience theories, showing that cognitive principles derived from animal studies can be powerful and practical engineering solutions. By bridging the gap from perception to decision-making, this work embraces a third paradigm for robotic autonomy: one that is neither purely proximal nor distal, neither explicitly programmed nor fully learned, but instead emerges from interpretable, dynamic, and resilient bio-inspired computation.

\section{Methods}\label{sec11}

In this section, we outline the key components of the proposed neuromorphic vision-to-control pipeline and present supporting analysis illustrating its decision-making mechanism. We begin by introducing the model formulation in Section~\ref{ref:model_form} and its discrete-time implementation in Section~\ref{sec:discretetime}. Section~\ref{sec:coarse_model} derives a reduced formulation of the model for $k$ clustered inputs corresponding to $k$ distinct targets, followed by a discussion of the timescale separation between neural dynamics and motion and bifurcation analysis in Section~\ref{sec:timescale_bif}. Finally, we describe the simulation environment and experimental setup in Sections~\ref{sec:sim_env} and~\ref{sec:experiment}.
\subsection{Model formulation}\label{ref:model_form}

We consider a robot moving in three dimensions, equipped with an onboard forward-facing RGB camera. We assume that the class of visual features that characterize the targets (e.g., color, shape, optical flow) is known.
For an image of resolution $(H,W)$, each pixel is assigned a value $u_i \geq 0$ representing the quantity of interest. Larger values of $u_i$ indicate higher confidence that pixel $i$ contains information about a target. We define the input vector $\mathbf{u} = [u_1,\dots,u_{H\times W}]^{\top}$ that contains goal information across all $H\times W$ pixels, as well as the diagonal input matrix \mbox{$\mathcal{\widehat{U}} = \operatorname{diag}(\mathbf{u}) \in \mathbb{R}^{(H \times W) \times (H \times W)}$}. The $\mathcal{\widehat{U}}$ matrix can be constructed using various image-processing techniques, such as segmentation, object detection, or neural implicit representations. If the number of pixels in the image exceeds the number of available neurons $k$,  the $\mathcal{\widehat{U}}$ matrix should be downsampled to a lower-dimensional matrix $\mathcal{U}\in \mathbb{R}^{k \times k}$, which is then passed into the neural dynamics defined in \eqref{eq:ndt_vector}. For example, in our experiments, we use an image resolutions of 144 by 256 pixels, which are downsampled via bilinear interpolation to a neural population of 36 by 64.

The image space is represented by $k$ populations of neurons whose states evolve in continuous time. Let $\mathbf{n}=[n_0, n_1, \hdots , n_k]^{\top}$ denote the vector of neural activity firing rates. Each neuron $i$ is associated with a directional unit vector $p_i \in \mathbb{R}^3$, collected in the $3\times k$ matrix $\mathcal{P} = [ p_1, \dots, p_k]$. The neural firing rate $n_i\in[0,1]$ represents the strength of preference of neural population $i$ for movement in direction $p_i$ in response to evidence from the inputs.  The neural dynamics, driven by the visual input encoded in the matrix $\mathcal{U}$,
evolve according to
\begin{equation}
    \frac{d\mathbf{n}}{dt} = \Pi_{\Delta}(f(\mathbf{n})), \quad f(\mathbf{n})= - \mathbf{n} +S\big(W(\mathbf{u})\mathbf{n}\big), \quad W(\mathbf{u}) = \mathcal{U}^{\top}\mathcal{P}^{\top} \mathcal{P}\mathcal{U},\label{eq:ndt_vector}
\end{equation} 
where $W(\mathbf{u})$ is an input-driven neural coupling matrix constructed as a Gram matrix of the input-weighted direction vectors in $\mathcal{P}$. The structure of $f(\mathbf{n})$ takes inspiration from mean-field neural models of animal spatial decision-making \cite{sridhar2021geometry} and recently proposed input-driven Hopfield and firing rate networks \cite{betteti2025input}. The map $\Pi_{\Delta}:\mathbb{R}_k^{+} \to \Delta_{k-1}$ in \eqref{eq:ndt_vector} constrains the evolution of neural activity to the simplex $\Delta_{k-1} = \{ \mathbf{n} \ \text{s.t.} \ n_i \in [0,1] \ \text{for all } i \ \text{and} \ \sum_{i = 1}^k n_i = 1 \}$, which we prove rigorously in Section 2 of the Supplementary Materials. Explicitly, this map is defined as 
\begin{equation}
    \Pi_{\Delta}(f(\mathbf{n})) = f(\mathbf{n}) - \left(\mathbf{1}^T f(\mathbf{n})\right) \mathbf{n}. \label{eq:simplex_projection}
\end{equation}
The nonlinear activation function $S(\cdot)$ in \eqref{eq:ndt_vector} is the sigmoid  
\begin{equation}
    S(x) = \frac{a}{1 + e^{- \alpha x}},
    \label{eq:sigmoid}
\end{equation}
where $a>0$ sets the upper bound of the saturation and $\alpha > 0$ is a coupling gain that tunes its slope.

The velocity of the robot in its body frame is selected by the neural activity \eqref{eq:ndt_vector} as described by
\begin{equation}
     \mathbf{v} = v_0 \mathcal{P}\tilde{\mathbf{n}}, \label{eq:v}
\end{equation}
where $v_0$ is a positive speed scaling constant and  $\tilde{\mathbf{n}}$ is the thresholded neural activity vector 
\begin{equation}
\tilde{\mathbf{n}} = [\tilde{n}_0, \tilde{n}_1, \dots, \tilde{n}_k]^{\top},  \quad 
\tilde{n}_i =
\begin{cases}
n_i, & \text{if } n_i > g_{\text{threshold}}, \\
0,   & \text{otherwise}.
\end{cases} \label{eq:thresholded_n}
\end{equation}

\subsection{Discrete-time neural dynamics model}
\label{sec:discretetime}
In the robot control logic, the simplex-constrained neural dynamic model \eqref{eq:ndt_vector},\eqref{eq:simplex_projection} is implemented in two discrete steps performed at every iteration of the algorithm: a neural activity update step and a normalization step. Let the discrete-time neural state be $\mathbf{n}_\ell = \mathbf{n}(\ell \Delta t) \in \Delta_{k-1}$ where $\ell \in \mathbb{N}$ is an index and $\Delta t > 0$ is a discretization step size. The two-step neural update reads
\begin{subequations} \label{eq:neural_dynamics_DT}
    \begin{align}
        \widehat{\mathbf{n}}_{\ell+1} & = \mathbf{n}_{\ell} + \Delta t \ f(\mathbf{n}_{\ell}) = (1 - \Delta t) \mathbf{n}_\ell + \Delta t \ S\big( W(\mathbf{u}) \mathbf{n}_{\ell} \big), \label{eq:discrete_state_update}\\
        \mathbf{n}_{\ell + 1} &= \frac{\widehat{\mathbf{n}}_{\ell+1}}{\mathbf{1}^{\top} \widehat{\mathbf{n}}_{\ell+1}} \label{eq:discrete_normalization},
    \end{align}
\end{subequations}
where $W(\mathbf{u})$, $S$ are defined as in \eqref{eq:ndt_vector},\eqref{eq:sigmoid}. To establish equivalence between the discrete-time model \eqref{eq:neural_dynamics_DT} and the continuous-time model \eqref{eq:ndt_vector},\eqref{eq:simplex_projection}, observe that it can equivalently be expressed as 
\begin{equation} \label{eq:neural_DT_onestep}
    \mathbf{n}_{\ell + 1} = \frac{\mathbf{n}_{\ell} + \Delta t \ f(\mathbf{n_\ell})}{\mathbf{1}^{\top} \big(\mathbf{n}_{\ell} + \Delta t \ f(\mathbf{n_\ell})\big) } = \frac{\mathbf{n}_{\ell} + \Delta t \ f(\mathbf{n_\ell})}{1 + \Delta t \ \mathbf{1}^{\top}  f(\mathbf{n_\ell}) } 
\end{equation}
where the second equivalence above uses the fact that $\mathbf{n}_{\ell} \in \Delta_{k-1}$. Taking a Taylor expansion of \eqref{eq:neural_DT_onestep} around $\Delta t = 0$ we get 
\[
    \mathbf{n}_{\ell+1} = \mathbf{n}_{\ell} + \Delta t \left( f(\mathbf{n}_{\ell}) - \big(\mathbf{1}^{\top} f(\mathbf{n}_{\ell}) \big)\mathbf{n}_{\ell} \right) + \mathcal{O}\big(\Delta t^2\big)
\]
or equivalently, rearranging the terms, 
\begin{equation} \label{eq:neural_DT_differential}
    \frac{\mathbf{n}_{\ell+1} - \mathbf{n}_{\ell}}{\Delta t} = f(\mathbf{n}_{\ell}) - \big(\mathbf{1}^{\top} f(\mathbf{n}_{\ell}) \big)\mathbf{n}_{\ell} + \mathcal{O}(\Delta t).
\end{equation}
In the limit $\Delta t \to 0$, \eqref{eq:neural_DT_differential} recovers the neural ODE \eqref{eq:ndt_vector},\eqref{eq:simplex_projection}. Note that without the normalization step \eqref{eq:discrete_normalization}, the state update \eqref{eq:discrete_state_update} is exactly an Euler discretized, input-driven generalization of the mean field decision-making model proposed to describe spatial movement of animals in \cite{sridhar2021geometry}. In our numerical experiments, we found the normalization step \eqref{eq:discrete_normalization} to be particularly important for generating well-behaved decision-making with high-dimensional state spaces. Without normalization to the simplex, calibration of the saturation bound and slope parameters $a,\alpha$ of the activation function \eqref{eq:sigmoid} for a variety of input profiles encountered during navigation is challenging. 

\subsection{Coarse-grained model for clustered inputs}
\label{sec:coarse_model}
In this section, we show that for $r < k$ small input clusters with binary inputs, the neural decision model \eqref{eq:ndt_vector},\eqref{eq:simplex_projection} approximately reduces to an $r$-dimensional effective ODE. This limit corresponds to absolute knowledge of target positions, such as in the case of an available map and accurate localization. The effective coarse-grained model we will derive is used in the benchmarking experiments in Section \ref{Sec2_1}.

Let $C_{s} \subseteq \{1, \dots, k\}$ be the set of pixel indices associated with cluster $s \in \{1, \dots, r\}$ of isolated visual inputs. Associated with each cluster $s$ we can define the total neural activation variable and average target position vector as 
\[
    \Bar{n}_s = \sum_{j \in C_s} n_j, \quad \bar{\mathbf{p}}_s = \frac{1}{|C_s|} \sum_{j \in C_s} \mathbf{p}_j,
\]
with the vector of neural activations for input clusters $1$ through $r$ being \mbox{$\bar{\mathbf{n}} = [\bar{n}_1, \dots, \bar{n}_r] \in \Delta_{r-1}$}. We assume that each cluster of inputs subtends a small angle in the visual field and therefore $\|\bar{\mathbf{p}}_s\| \approx 1$,  meaning that for any $i \in C_s$, $j \in C_{\ell}$, 
\[
    \mathbf{p}_i^{\top} \mathbf{p}_j \approx \bar{\mathbf{p}}_s^{\top} \bar{\mathbf{p}}_{\ell}.
\]
Furthermore, we assume that each pixel input $u_i \in \{0,1\}$ where $u_i = 0$ indicates direction $\mathbf{p}_i$ is not associated with a target, while $u_i = 1$ means it is associated with a target. Under these assumptions, observe that 
\begin{equation}
    W_{ij}(\mathbf{u}) = u_i u_j \mathbf{p}_i^{\top} \mathbf{p}_j \approx \begin{cases}\bar{\mathbf{p}}_s^{\top}\bar{\mathbf{p}}_{\ell} \quad \text{if} \ i \in C_s \ \text{and} \ j \in C_{\ell} \ \text{for some } s,\ell \in \{1,\dots,r\}, \\
        0 \quad \quad \ \ \text{otherwise}.
    \end{cases}
\end{equation}
Then for any pixel associated with an input cluster, $i \in C_s$, the input to the saturation function $S$ in \eqref{eq:ndt_vector} is
\begin{equation}
    (W(\mathbf{u}) \mathbf{n})_i = \sum_{\ell = 1}^r \sum_{j \in C_{\ell}} \mathbf{p}_i^{\top} \mathbf{p}_j n_j \approx \sum_{\ell = 1}^r \bar{\mathbf{p}}_s^{\top}\bar{\mathbf{p}}_{\ell} \sum_{j \in C_{\ell}} n_j = \sum_{\ell = 1}^r \big(\bar{\mathbf{p}}_s^{\top}\bar{\mathbf{p}}_{\ell}\big)\bar{n}_\ell. \label{eq:sat_input_cluster}
\end{equation}
For any pixel $i$ not associated with an input cluster, i.e. for which $u_i = 0$, $(W(\mathbf{u}) \mathbf{n})_i = 0$ and the neural dynamics simplify to 
\[
    \frac{dn_i}{dt} = - n_i + \frac{a}{2} - (\mathbf{1}^{\top} f(\mathbf{n})) n_i = \frac{a}{2}- \big(1 + (\mathbf{1}^{\top} f(\mathbf{n}))\big) n_i = \frac{a}{2} - \big( \mathbf{1}^{\top} S(W(\mathbf{u}) \mathbf{n})\big) n_i
\]
where the last step used the fact that $\mathbf{1}^{\top}\mathbf{n} = 1$ due to the simplex constraint.
Let $I_0$ be the set of all pixel indices for which $u_i = 0$. Then at an equilibrium $\mathbf{n} = \mathbf{n}^*$, every zero-input neuron $i \in I_0$ has identical steady-state activity level 
$n_i^* = n_0^* := \frac{a \beta^*}{2}$
where $\beta^* := \mathbf{1}^T S(W(\mathbf{u} )\mathbf{n}^*)$. Since $\beta^* > |I_0| \frac{a}{2}$ and $\beta^* < k a$, we can bound the steady-state activity as $\frac{1}{2k} < n_0^* < \frac{1}{|I_0|}$. Then we see that as the number of pixels $k$ grows large and the number of zero-input pixels grows large, $n_0^* \to 0$.

To derive the coarse-grained neural model we will consider the evolution of the total neural activity associated with input cluster $s$ under the assumption $k \gg 1$, $|I_0| \gg 1$ so that $n_0^* \approx 0$. The evolution of $\bar{n}_s$ is governed by
\begin{equation}
    \frac{d \bar{n}_s}{dt} = \sum_{j \in C_s} \frac{d n_j}{dt} = \sum_{j \in C_s} f_j(\mathbf{n}) - (\mathbf{1}^{\top} f(\mathbf{n})) \sum_{j \in C_s} n_j. \label{eq:cluster_dyn}
\end{equation}
Using \eqref{eq:sat_input_cluster}, observe that 
\[
    \sum_{j \in C_s} f_j(\mathbf{n}) = - \sum_{j \in C_s} n_j  + \sum_{j \in C_s} S\Big( \big( W(\mathbf{u}) \mathbf{n}\big)_{j} \Big) \approx - \bar{n}_s + |C_s| S\left(\big( \bar{\mathcal{P}}^{\top}\bar{\mathcal{P}} \bar{\mathbf{n}}\big)_{s}\right) 
\]
where $\bar{\mathcal{P}} = [\bar{\mathbf{p}}_1, \dots, \bar{\mathbf{p}}_r]$. With this approximation, the clustered neural activity dynamics \eqref{eq:cluster_dyn} can be expressed as 
\begin{equation}
    \frac{d\bar{\mathbf{n}}}{dt} = \Pi_{\Delta} \big(\bar{f}(\bar{\mathbf{n}})\big), \quad \bar{f}(\bar{\mathbf{n}}) = - \bar{\mathbf{n}} + C S\left( \bar{\mathcal{P}}^{\top}\bar{\mathcal{P}} \bar{\mathbf{n}}\right) \label{eq:coarse_grained_general}
\end{equation}
where $C = \operatorname{diag}(|C_1|, \dots, |C_r|)$ and $\Pi_{\Delta}$ is defined as in \eqref{eq:simplex_projection}. If we further assume that the input clusters have equal size, $|C_i| = c $ for all $i \in \{1, \dots, r\}$, i.e. the options are perceived to have equal quality, then the coarse-grained model can be expressed as 
\begin{equation}
    \frac{d\bar{\mathbf{n}}}{dt} = \Pi_{\Delta} \big(\bar{f}(\bar{\mathbf{n}})\big), \quad \bar{f}(\bar{\mathbf{n}}) = - \bar{\mathbf{n}} + S\left( \bar{\mathcal{P}}^{\top}\bar{\mathcal{P}} \bar{\mathbf{n}}\right) \label{eq:coarse_grained_equal_input}
\end{equation}
where the cluster size parameter $c$ is absorbed into the upper bound parameter $a$ in the saturating function definition \eqref{eq:sigmoid}. The benchmarking experiments in Section \ref{Sec2_1} were performed with the coarse-grained equal-input model \eqref{eq:coarse_grained_equal_input}.

\subsection{Timescale separation and bifurcation analysis}
\label{sec:timescale_bif}

In the full vision-guided navigation problem, the robot's velocity is directly informed by the neural state as $\mathbf{v} = v_0 \mathcal{P}\tilde{\mathbf{n}}$, which in turn depends on visual inputs that change as the robot moves through space. This creates a coupled dynamical system in which neural activity and spatial position of the agent coevolve. However, in typical operating regimes, the neural dynamics integrates visual evidence on a much faster timescale than that of the spatial movement dynamics of the robot. Under this timescale separation, equilibria of the neural dynamics \eqref{eq:ndt_vector} induce a slow manifold that is tracked by the neural state as the robot moves \cite{kuehn2015multiple}. Therefore understanding the embedded decision-making in the moving agent requires understanding the structure of equilibria in the neural decision model. In this section, we will illustrate with a simple example that transitions between qualitatively different movement behaviors, i.e. \textit{decisions}, occur when the visual input drives the neural dynamics through a \textit{bifurcation point}. We will then use this insight to define a theoretically-justified heuristic that allows us to identify such points of decision-making in simulations and hardware experiments.  

\subsubsection{Bifurcation analysis of two-option coarse-grained ND model}

To illustrate the core decision-making mechanism of the neural dynamics, we study analytically the simplest scenario of making a decision between two alternatives. To do this, we focus our attention on the coarse-grained model \eqref{eq:coarse_grained_general} derived in Section~\ref{sec:coarse_model}. Specialized to the two-option case, and under the simplex constraint $\bar{n}_2 = 1 - \bar{n}_1$, this model reduces to a one-dimensional ordinary differential equation
\begin{equation}
    \frac{d\bar{n}_1}{dt} = |C_1|(1-\bar{n}_1) S\big((1- c_{12})\bar{n}_1  + c_{12}\big) - |C_2|\bar{n}_1 S\big((c_{12}-1)  \bar{n}_1 + 1\big)
    \label{eq:pitchfork_unfold_n1}
\end{equation}
where $c_{12} := \cos(\theta_{12}) = \mathbf{p}_1^T\mathbf{p}_2$ and $\theta_{12}$ is angular distance between the two targets as perceived by the moving agent at a particular point in space. In Section 3 of the Supplementary Materials document we carry out a detailed mathematical analysis of \eqref{eq:pitchfork_unfold_n1}. In this section we summarize the key takeaways from this analysis and interpret them in the context of autonomous spatial decision-making. 

Since spatial movement in the direction of the two targets will increase the visual angle $\theta_{12}$ and therefore decrease $c_{12}$ between its extreme values of $1$ and $-1$, it is natural to consider the parameter $\mu = 1 - c_{12} \in [0,2]$ and ask how the number and stability of equilibrium solutions of \eqref{eq:pitchfork_unfold_n1} change with variation in $\mu$ as it increases from 0 to 2. In other words, we are interested in \textit{bifurcations of equilibria} in the model \eqref{eq:pitchfork_unfold_n1} parametrized by $\mu$, which can be studied using tools from dynamical systems theory \cite{guckenheimer2013nonlinear,Golubitsky1985}. When $\mu$ is close to 0, the two targets are very close in the agent's visual field  with $\theta_{12} \approx 0^{\circ}$. When $\mu$ is close to 2, $\theta_{12}$ is close to $180^{\circ}$, with the targets on opposite sides of the agent. At some intermediate critical angle $\theta^*$, and therefore critical value $\mu^*$ we expect the model to exhibit a bifurcation in which new equilibria that correspond to a decision emerge as the parameter $\mu$ is varied. Agents cross this neural bifurcation point as they move towards targets and increase their visual angle.

\begin{figure}
    \centering
    \includegraphics[width=0.99\linewidth]{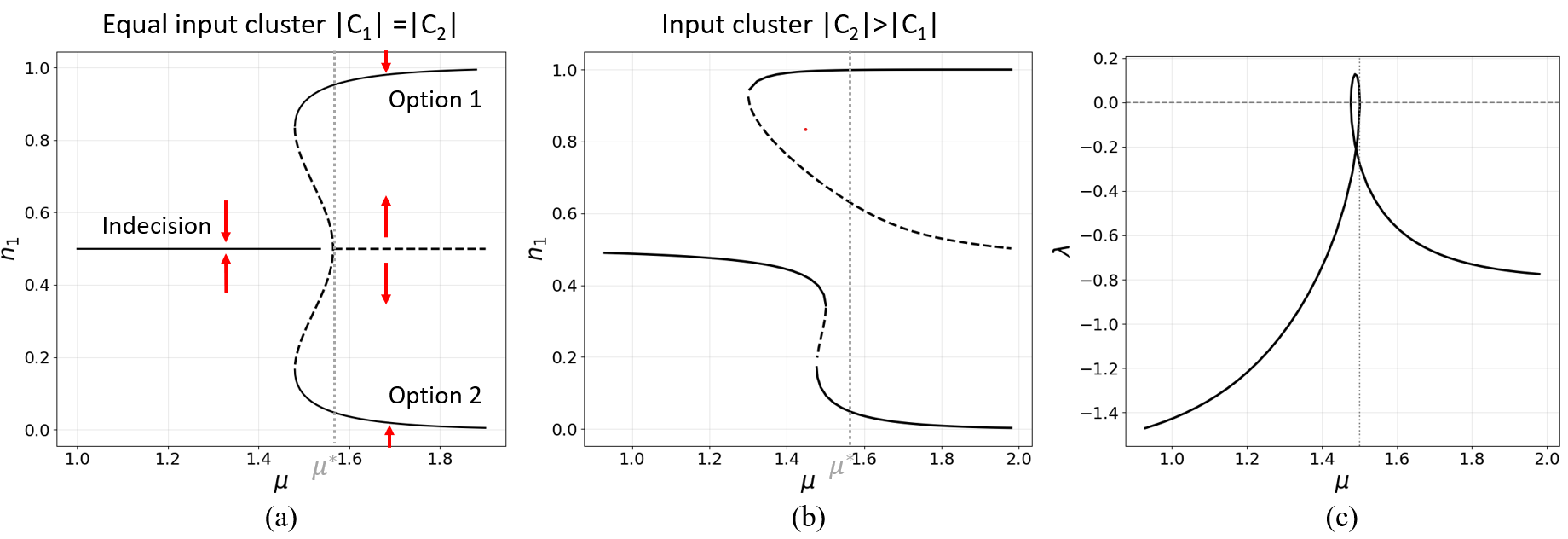}
    \caption{\textbf{Bifurcation structure of the two-option coarse-grained ND model.}
    \textbf{(a)}  Pitchfork bifurcation diagram with equal input cluster.  Bifurcation branches  of stable (unstable) equilibria are depicted as solid (dashed) lines. The gray vertical line marks the symmetric bifurcation point $\mu^{*} = 1.56$, where the system transitions from an indecision state to two symmetric decision states (Option 1 or Option 2).
     \textbf{(b)} Bifurcation diagram under unequal input clusters. The same gray vertical line from (a) is shown to highlight that asymmetry leads to an ultrasensitive response near the symmetric bifurcation point $\mu^{*}$, collapsing the symmetric pitchfork into an imperfect, biased decision structure.
     \textbf{(c)} Jacobian eigenvalue $\lambda$ 
     along the lower branch in (b). The bifurcation diagrams in (a) and (b) are computed with Julia BifurcationKit.jl \cite{veltz:hal-02902346}. Model parameters: $a = 0.8, \alpha = 6$; for (a), $|C_1| = |C_2| = 1$; for (b), $ |C_1|= 1, |C_2| = 1.3$.}
    \label{fig:bifurcation_analysis}
\end{figure}

When the two input clusters have equal size, i.e. $|C_1| = |C_2|$, the indecision state $\bar{n}_1 = \bar{n}_2 = \frac{1}{2}$ is always an equilibrium of \eqref{eq:pitchfork_unfold_n1}. At this equilibrium, the agent's instantaneous velocity will point in the average direction of the two targets. Linear stability analysis reveals that for sufficiently large coupling gain ($\alpha > 2$), there exists a critical value $\mu^*$ defined implicitly by 
\begin{equation}
    \frac{\alpha}{4} \mu^* \tanh\left( \frac{\alpha}{4}(2 - \mu^*)\right) -  \frac{1}{4} \alpha \mu^* + 1 = 0 \label{eq:bif_condition}
\end{equation}
at which the indecision equilibrium  loses stability in a pitchfork bifurcation at which two additional solution branches emerge in a neighborhood of the bifurcation point - see Fig. \ref{fig:bifurcation_analysis}a. For $\mu > \mu^*$ two new equilibria exist: one with $\bar{n}_1 > 0 $ which corresponds to a decision to bias the direction of movement towards target 1, and one with $\bar{n}_1 < 0$ which corresponds to a decision to bias the direction of movement towards target 2. As shown in \eqref{eq:bif_condition}, the value of the parameter $\mu^*$ and therefore of the critical angle $\theta^*$ at the bifurcation point depends on the choice of the slope parameter $\alpha$ of the neural saturation. We find that the smallest possible critical value for the visual angle is $\theta^* \approx 124^{\circ}$ which occurs at the value $\alpha \approx 5.87$, with smaller or larger choices of $\alpha$ leading to a larger critical angle. However, we note that this lower bound on the bifurcation angle is specific to the two-option coarse-grained model, and  for the multidimensional model \eqref{eq:ndt_vector} with the more realistic assumption that clusters of stimuli increase in size as an agent gets closer to them in physical space,  a different set of calculations must be performed in order to find this estimate. In the simulation studies shown in previous sections, bifurcations were observed for the full-dimensional model at visual angles smaller than $124^{\circ}$.

When the two input clusters have unequal size, i.e. when $|C_1| \neq |C_2|$, the asymmetry $\bar{C} = \frac{|C_1| - |C_2|}{|C_1| + |C_2|}$ acts as an \textit{unfolding parameter} that biases the bifurcation diagram towards one of the two decision branches, see Fig.\ref{fig:bifurcation_analysis}b. Near the bifurcation point $\mu = \mu^*$, the decision is ultrasensitive to the cluster asymmetry and even small differences between $|C_1|$ and $|C_2|$ are amplified into a decisive commitment towards the larger cluster. This ultrasensitivity manifests itself as a characteristic signature in the dominant Jacobian eigenvalue along the continuous branch of solutions in the bifurcation diagram, as pictured in Fig.\ref{fig:bifurcation_analysis}c. The eigenvalue spikes towards zero as the system passes through a neighborhood of $\mu^*$ before returning to a more negative value. This spike reflects a ``ghost'' of the symmetric pitchfork bifurcation point that persists even when the symmetry is broken. Since any realistic scenario will have some asymmetry in visual inputs, either from true asymmetry in target sizes or due to inevitable sensor and discretization noise, this eigenvalue signature provides a practical criterion for detecting decision points in simulations and experiments. An observed spike in the dominant eigenvalue at a nearby neural equilibrium signals that the agent is passing through the critical region where small input differences are being amplified into a decisive choice to sharply change its direction of movement towards a subset of visible targets.

Through the lens of this bifurcation analysis, we can understand the decision behavior of an agent's movement through space. As an agent approaches two targets, aligning its heading with their average direction, the angular separation $\theta_{12}$ increases, causing $\mu$ to grow over time. The agent thus traverses the bifurcation diagram, e.g. as pictured in Fig.\ref{fig:bifurcation_analysis}a,b, from left to right. Initially, when targets appear close together in the visual field, $\mu < \mu^*$ and the agent's neural state closely tracks the indecision equilibrium branch in the diagram. As the agent approaches and crosses $\mu^*$, it enters the ultrasensitive regime where any asymmetry - whether from slightly different target sizes, distances, or visual salience, is amplified strongly into a decisive commitment toward one target as the continuous branch of solutions in Fig. \ref{fig:bifurcation_analysis}b is tracked by the neural state of the agent. This mechanism provides a principled explanation for the characteristic trajectories observed in our experiments, in which agents sharply transitioned from movement in the average direction of observed targets to approaching a selected subset of the targets.

\subsubsection{Heuristic to detect decision points in general ND model}

As  illustrated by the low-dimensional analysis in the previous section, transition of an agent through points of \textit{bifurcation} or \textit{decision} correspond to characteristic spikes in the real part of the leading eigenvalue of the Jacobian matrix of the neural dynamics, evaluated at a nearby equilibrium point. Conjecturing that the high-dimensional system exhibits analogous bifurcation structure to its low-dimensional counterpart, this eigenvalue-based heuristic thus provides a practical tool for identifying decision points in simulations and experiments with the full neural dynamics model. The Jacobian of the neural dynamics in \eqref{eq:ndt_vector} is given by
\begin{subequations} \label{eq:Jac}
\begin{gather}
    \frac{\partial f}{\partial \mathbf{n}}(\mathbf{n}) = -I + a \alpha \  \text{diag}(\operatorname{sech}^2( \alpha \mathcal{U}^{\top}\mathcal{P}^{\top} \mathcal{P}\mathcal{U}\mathbf{n})) \mathcal{U}^{\top}\mathcal{P}^{\top} \mathcal{P}\mathcal{U},\\
    J(\mathbf{n}) = (I - \mathbf{n} \mathbf{1}^{\top}) \frac{\partial f}{\partial \mathbf{n}} - \left( \mathbf{1}^{\top} f(\mathbf{n}) \right) I.
\end{gather}
\end{subequations}
Given a particular choice of input matrix $\mathcal{U}$ and parameters $a,\alpha$, let $\lambda_{1}$ be the eigenvalue of $J(\mathbf{n})$ with largest real part, i.e. its dominant eigenvalue. 
Recall that due to the timescale separation between neural and spatial dynamics, we expect the neural state $\mathbf{n}(t)$ to closely track equilibria of the neural dynamics with induced by the corresponding input $\mathbf{u}(t)$.
Therefore it is sufficient to track $\operatorname{Re}(\lambda_1(t))$ of \eqref{eq:Jac} evaluated at the current neural state $\mathbf{n}(t)$ to detect these decision spikes.  
We verify this conjecture in the simulation studies in Section~\ref{Sec2_2}, where in Fig.~\ref{fig:sim_py}c we show how at an observed decision point in the spatial trajectories, $\lambda_{1}(t)$ indeed exhibits a spike which indicates ultra-sensitivity to input differences.
Note that the stabilized activity
patterns depend on the distribution of the visual input $\mathcal{U}$.
Consequently, changes in the vision---such as a target leaving the FOV or the robot passing a target---can also produce
transient spikes in $\operatorname{Re}(\lambda_{1})$, as observed in
Fig.~\ref{fig:sim_py}d at the moments indicated by the black dashed line.
Orange dashed lines in the same plot indicate points of decision that are driven by bifurcation in the neural dynamics.

\subsection{Simulation environment\label{sec:sim_env}}

In the simulation experiments, we employ a plug-in API in AirGen \cite{vemprala2023gridplatformgeneralrobot} that receives body-frame velocity commands generated by the proposed pipeline. A Proportional-Derivative controller regulates the yaw angle to ensure that the onboard camera consistently maximizes target occupancy within its FOV. The quadrotor’s state is recorded for visualization and analysis purposes but are not used for feedback control.

The threshold parameter $g_{\text{threshold}}$ in \eqref{eq:thresholded_n} is implemented as the instantaneous neural activity of the center pixel. In the absence of visual input, this center pixel maintains the highest baseline activity, and thus provides an adaptive estimate of the background activity level. Any pixel whose activity exceeds this center value is therefore interpreted as providing strong evidence that its corresponding direction points toward a target.

We study two simulation scenarios: static navigation in the blocks environment and multi-target tracking in the beach environment. A color-based masking technique is applied to the RGB images to extract pixels corresponding to the target color (green). For autonomous sequential navigation, a local collision-avoidance controller maintains safe and continuous operation by activating whenever a sufficient proportion of pixels in the FOV exhibit strong target evidence.

In  multi-target tracking scenario, three Astro quadrotors are considered as targets as shown in Fig.~\ref{fig:grid2_2}. Their trajectories are predefined.  Logging the states of all quadrotors slightly reduces the frame update rate to $\sim10$ Hz. To extract motion-based target evidence from the optical flow, we derive a per-pixel representation by identifying localized deviations from the background motion field. Let $\mathbf{I}(x,y)$ denote the raw optical flow and $\mathbf{I}_{bg}(x,y)$ the background flow estimated via a robust affine fit. The residual field $\mathbf{r}(x,y)=\mathbf{I}(x,y)-\mathbf{I}_{bg}(x,y)$ suppresses global camera-induced motion and large-scale background drift, isolating only scene elements whose motion diverges from the dominant image-plane pattern. Foreground motion is then detected through local statistical deviation: for each pixel, we compute the local mean and standard deviation of the normalized residual within an $5\times 5$ neighborhood. The deviation from this local baseline is quantified by comparing the magnitude of the residual at each pixel to the local standard deviation, and a pixel is classified as containing meaningful motion whenever this ratio exceeds a fixed threshold. The resulting motion-evidence field provides a spatially localized estimate of dynamic activity, which is subsequently pooled into the neural population as a per-pixel evidence signal.

\subsection{Experimental setup\label{sec:experiment}}

The experiments have been conducted in an arena of $5.5 \times 8.5$ m. The robot platform is a custom quadrotor that has a distance between its center of mass and the center of a rotor of $0.15$ m. Besides, the quadrotor has a weight of $0.9$ kg, including the battery. The battery is a $4$-cell lithium ion battery that provides a flight autonomy of $6-8$ min, depending on the flight regime. All the algorithms run on board the quadrotor; specifically, the quadrotor uses a Raspberry Pi 4 with Ubuntu 22.04 as its unique processor. The algorithms are implemented in a ROS2 humble python node that runs at consistent $10$ Hz across experiments. Images are captured with a camera characterized by a native resolution of $8$ megapixels and a FOV of $62.2 \times 48.8$ degrees. The low-level controller is the Linear Quadratic Regulator module in Frejya \cite{shankar2021freyja}. Data for visualization and statistical analysis has been acquired with a motion capture system.

\section*{Data availability}

All data needed to evaluate the conclusions in the paper are present in the paper or the Supplementary Materials.

\section*{Code availability}

The code to reproduce the results in simulation can be found at \url{https://github.com/ChuweiW/Bio-inspired-navigation}. Supplementary videos can be found at \url{https://drive.google.com/drive/folders/1tcKIke0mA300ASGlLL5C98jqLMH4uUbk?usp=drive_link}.

\section*{Acknowledgments}

The authors are grateful to Alessio Franci for helpful constructive feedback on an early draft of the manuscript. We also thank Ajay Shankar for his assistance with the quadrotor experiments.

\section*{Author contributions}

A.B. and E.S. proposed the initial idea of the research. C.W. and A.B. developed the modeling approach and analysis. C.W. implemented the proposed approach and designed the simulated experiments. E.S. conducted the hardware quadrotor experiments. 
A.P. provided technical feedback on method motivation, implementation, and evaluation.
The first manuscript draft was written by C.W., E.S. and A.B, with all authors contributing to subsequent drafts. A.B. and A.P. provided the funding. 

\section*{Funding}
E.S. and A.P. were partially funded by a Leverhulme Trust Research Project Grant. 

\section*{Competing Interests}

There are no competing interests to declare.

\printbibliography[title={References}]
\end{refsection}

\clearpage
\newpage

\begin{refsection}
\beginsupplement
\renewcommand{\figurename}{Supplementary Figure}
\renewcommand{\tablename}{Supplementary Table}

\section*{\LARGE Supplementary Material}

\section{Benchmarks}
\label{Supp:Benchmark}

We denote as $z_k = [x, v_x, y, v_y]^{\top}$ the configuration vector of the robot at time $k$, with $\mathbf{p}_k = [x, y]$ the position. The position of each target is denoted as $\mathbf{g}_i = [x_{t,i}, y_{t,i}]$. The robot is assumed to have access to its odometry and every target's position. Targets are stationary. No obstacles are present in the environment. The robot is considered to have successfully reached target $i$ if it arrives to a configuration within a small tolerance $\epsilon >0$, i.e., $||\mathbf{p}_k - \mathbf{g}_i|| \leq \epsilon$. The number of targets $K$ is equal to 2 or 3 depending on the scenario.

\subsection{Neural Dynamics}

To ensure fair comparison with respect to the benchmarked approaches, our method in Sec. 2.1 (ND) is simplified as follows. The model dynamically adapts the neural activity rates $\bar{\mathbf{n}}$ based on the relative position $\bar{\mathcal{P}} = [\bar{\mathbf{p}}_1, \bar{\mathbf{p}}_2] = \left[\frac{(\mathbf{g}_1 - \mathbf{p}_t)^{\top}}{\lVert \mathbf{g}_1 - \mathbf{p}_t \rVert}, \frac{(\mathbf{g}_2 - \mathbf{p}_t)^{\top}}{\lVert \mathbf{g}_2 - \mathbf{p}_t \rVert}\right]$ that contains unit vectors pointing from the agent to targets 1 and 2,  when $K = 2$ and $\mathcal{P} = [\bar{\mathbf{p}}_1, \bar{\mathbf{p}}_2, \bar{\mathbf{p}}_3] = \left[\frac{(\mathbf{g}_1 - \mathbf{p}_t)^{\top}}{\lVert \mathbf{g}_1 - \mathbf{p}_t \rVert}, \frac{(\mathbf{g}_2 - \mathbf{p}_t)^{\top}}{\lVert \mathbf{g}_2 - \mathbf{p}_t \rVert}, \frac{(\mathbf{g}_3 - \mathbf{p}_t)^{\top}}{\lVert \mathbf{g}_3 - \mathbf{p}_t \rVert}\right]$ when $K=3$. As a consequence, the neural dynamics used in Sec. 2.1 is 
\[
\frac{d \bar{\mathbf{n}}}{dt} = \Pi_{\Delta} \big(-\bar{\mathbf{n}} + S(\bar{\mathcal{P}}^{\top}\bar{\mathcal{P}}\bar{\mathbf{n}}) \big)
\]
where $S: \mathbb{R} \to (0,a)$ is a saturating function applied element-wise to its input vector and $\Pi_{\Delta}: \mathbb{R}_K^{+} \to \Delta_{K-1}$ maps the neural activity to the unit simplex.

\subsection{Potential Field}
The potential field (PF) algorithm uses an artificial potential field \cite{Khatib85PF} modulated by an attractive force $\mathbf{F}_{att}(\mathbf{p}_k)$ at position $\mathbf{p}$  
\[
\mathbf{F}_{att}(\mathbf{p}_k) = \sum_{i=1}^{T} (\mathbf{g}_i - \mathbf{p}_k)
\]
which pulls the robot toward the set of target positions.

\subsection{Model Predictive Control}
We compare our method with a model predictive control (MPC) framework using a point-mass model in 2D plane for the robot. We set the robot's dynamics as
\[
z_{k+1} = Az_k + Bu_k
\]
where $$A=\begin{bmatrix}
1 & \Delta t & 0 &0\\
0 & 1 & 0 & 0  \\
0 & 0 & 1 & \Delta t \\
0 & 0 & 0 & 1  \\
\end{bmatrix}, \qquad B=\begin{bmatrix}
0 & 0\\
\Delta t & 0\\
0 & 0\\
0 & \Delta t\\
\end{bmatrix}.$$ The sets of states and inputs are denoted by $\mathcal{Z}$ and $\mathcal{U}$, respectively.
Specifically, we consider two variants of MPC. The first formulation, denoted as MPC1, minimizes a weighted sum of squared distances to all targets. Let N be the prediction horizon. The path planning
can be formulated as an optimization problem:
\begin{equation}
\begin{aligned}
\min_{\mathbf{u}_{0:N-1}} \quad & \sum_{k=0}^{N-1} \sum_{\forall i \in T} \big( w_i\big(||\mathbf{p}_k - \mathbf{g}_i||^2\big) + w_u u_k^2 \big) + \phi \\
\textrm{s.t.} \quad & z_{k+1} = Az_k + Bu_k\\
  &z_{k+1}\in \mathcal{Z}, u_k\in \mathcal{U}   \\
  &k=0,1,2,..N-1   \\
\end{aligned}\label{eq:MPC1}
\end{equation}
where $w_i$ penalizes the distance to the $i$-th target and $w_u u_k^2$ penalizes large control inputs and   $\phi$ is the terminal cost:
\[
\phi = w_{\text{t}}\sum_{\forall i\in T} \|\mathbf{p}_N - \mathbf{g}_i\|^2 .
\]

To encourage decisive behavior, we introduce a smooth approximation of the minimum operator using the soft-min function. This second formulation, MPC2, is expressed as
\begin{equation}
\begin{aligned}
\min_{\mathbf{u}_{0:N-1}} \quad &
\sum_{k=0}^{N-1} \ell_k + \phi, \\
\text{s.t.} \quad & z_{k+1} = A z_k + B u_k, \\
& u_k \in\mathcal{U}, z_{k+1} \in \mathcal{Z} \\
& k=0,1,2,..N-1  
\end{aligned}
\label{eq:MPC2}
\end{equation}
The stage cost $\ell_k$ penalizes the robot’s distance to the targets and control effort
\[
\ell_k = \operatorname{smin}_{\tau}\big(\{\|\mathbf{p}_k - \mathbf{g}_i\|^2\}_{i\in T}\big) + w_u \|u_k\|^2,
\]
and  $\phi$ is the terminal cost
\[
\phi = w_{\text{t}} \operatorname{smin}_{\tau}\big(\{\|\mathbf{p}_N - \mathbf{g}_i\|^2\}_{i\in T}\big).
\]
The \emph{soft-min} function is defined by
\[
\operatorname{smin}_{\tau}(\{e_i\}) = -\tau \log \sum_{i} \exp(-e_i / \tau).
\]

\subsection{Reinforcement Learning}\label{subsec:rl}

The reinforcement learning policy is parameterized as a multi-layer perceptron (MLP) with two hidden layers of 256 neurons and hyperbolic tangent as activation function for all the layers including the output layer. The observation vector used as input to the policy is given by
$$
o_k = [z_k^{\top}, s_{k,1}^{\top}, s_{k,2}^{\top}, s_{k,3}^{\top}]^{\top},
$$
where $s_{k,i}^{\top} = [x_{t,i} - x_k, y_{t,i} - y_k, e_i]$. Furthermore, we assume that the maximum number of targets in the field of view of the robot is $3$. On the other hand, $e_i = 0$ if the target is active and $1$ otherwise. The output of the policy provides x-acceleration and y-acceleration. To train the policy, we use a version of the dispersion environment from Multi Particle Environments as implemented in the Vectorized Multi-Agent Simulator (VMAS) \cite{bettini2022vmas}. Specifically, we randomize the number of targets as either $T=2$ or $T=3$ with equal probability, the number of agents equal to $1$, the maximum number of steps per episode as $200$, and the reward function as follows:
$$
R_k = \begin{cases}
    1 & \quad \text{if } \exists i\in T \text{ s.t. } ||\mathbf{p}_k - \mathbf{g}_i||<d \text{ and } e_i=0
    \\
    -\min_{i} ||\mathbf{p}_k - \mathbf{g}_i|| & \quad \text{otherwise}
    \end{cases},
$$
with $d=0.1$ a distance threshold. The robot is trained in an arena of $[-1, 1] \times [-1, 1]$, so policy inputs and outputs are scaled after training to match the dimensions of the environment where the policy is deployed. We train the parameters of the policy using proximal policy optimization (PPO) \cite{schulman2017proximal} as implemented in BenchMARL \cite{bettini2024benchmarl}. All the parameters of the algorithm follows the default configuration of independent PPO, except that we set a batch size of $512$, a learning rate of $5\times 10^{-4}$ and $1\times 10^7$ training steps. Policy and critic share the same neural network.

\subsection{Benchmark results}

In addition to the results presented in the article, we collected the mean and standard deviation of per-step computation times for all planners in two-goal and three-goal setting, as summarized in Table~\ref{tab:comp_time1} and Table~\ref{tab:comp_time2} respectively. All benchmarks are performed on a Intel Core i9-14900K CPU, 64 GB RAM, Ubuntu 24.04 computer.
Both PF and ND model exhibit extremely low and consistent computational costs, highlighting their efficiency and scalability.
The RL policy executes a single forward pass through a fixed MLP at every control step without any iterative optimization or search. Its computational cost therefore scales only with the network size and input dimensionality which makes it significantly faster than MPC and only slightly higher than potential field or neural dynamics methods. 
In contrast, the MPC-based planners require significantly higher computation times---by several orders of magnitude---due to the repeated optimization performed at each step.
Their runtime is further influenced by parameters such as the prediction horizon length, number of targets, and the complexity of the cost function. For comparison, we also measure the computation time of the ND update using a neural population of size $64 \times 32$---sufficient to represent a field of view smaller than $180^{\circ}$. This configuration achieves per-step runtimes on the same order of magnitude as the RL network, with a mean computation time of $0.729$ms and a standard deviation of $0.091$ms over 30 runs. The reference RL network consists of an 11-dimensional input, two hidden layers of 256 units each, and a 2-dimensional output.

\begin{table}[ht!]
\centering
\caption{Computation efficiency comparison across navigation strategies derived from measured pre-reach statistics, averaged over 30 runs (2 targets). 
For each method, the mean per-step computation time, standard deviation of per-step compute time and the cumulative computation time required to reach a target are reported. 
A value of $\infty$ indicates that the robot did not reach any target. }
\begin{tabular}{lccc}
\toprule
\textbf{Method} & \textbf{Mean [ms]} & \textbf{Std [ms]} & \textbf{Total time [ms]} \\
\midrule
MPC1 (Weighted Sum, N=10)  & 52.16 & 9.27   & $\infty$ \\
MPC2 (Soft-min, N=10)      & 245.50 &  38.54 & 2059 \\
PF      & 0.01 & 0.00 & $\infty$ \\
RL     & 0.571 & 0.087 & 16.18 \\
ND  & 0.09 & 0.01 & \textbf{1.80} \\
\bottomrule
\end{tabular}
\label{tab:comp_time1}
\end{table}

\begin{table}[ht!]
\centering
\caption{Computation efficiency comparison across navigation strategies derived from measured pre-reach statistics, averaged over 30 runs (3 targets). 
For each method, the mean per-step computation time, standard deviation of per-step compute time and the cumulative computation time required to reach a target are reported. 
A value of $\infty$ indicates that the robot did not reach any target. }
\begin{tabular}{lccc}
\toprule
\textbf{Method} & \textbf{Mean [ms]} & \textbf{Std [ms]} & \textbf{Total time [ms]} \\
\midrule
MPC1 (Weighted Sum, N=15)  & 147.85 & 25.331  & $\infty$ \\
MPC2 (Soft-min, N=15)      & 662.107&  209.026& 1318 \\
PF      & 0.01 & 0.003 & $\infty$ \\
RL     & 0.503 & 0.126 & 8.6 \\
ND & 0.09 & 0.01 & \textbf{1.563} \\
\bottomrule
\end{tabular}
\label{tab:comp_time2}
\end{table}

\subsection{The decision-making process in the RL policy}\label{Supp_RL_analysis}

Among the benchmarked methods, the the RL policy is an entirely data-driven model. This implies the decision-making process of the RL policy cannot be interpreted uniquely from the structure of the MLP. As a consequence, training process and architectural choices can induce biases that bootstrap decision-making. To ensure that this is not the case, we conducted a series of statistical analysis over different variants of the RL policy and training configuration. Specifically, we study the evolution of three metrics over training, deploying each policy checkpoint in 200 random initial configurations, not necessarily symmetrical with respect to target positions, following the same specifications reported in Sec. \ref{subsec:rl}. The first metric (M1) computes the proportion of times the robot navigates to the closest target at the initial time step of the rollout. The second metric (M2) computes the distribution of closest targets at the initial time step of the rollout with respect to their appearance in the observation vector of the robot. The third metric (M3) computes the distribution of first approached target with respect to their appearance in the observation vector of the robot. An unbiased policy should always go to the initially closest target, as it is the optimal policy according to the reward structure, leading to M1=1. Besides, if M2 is a uniform distribution, M3 should be a uniform distribution, as the robot should not have any preference for targets as a function of their order in the observation vector.

As shown in Fig. \ref{fig:analysis_rl}, a policy trained following the specifications in Sec. \ref{subsec:rl} and the dispersion scenario in VMAS suffers from inductive biases. The agent, fast enough to cover the arena in a dozen time steps, does not learn to navigate to the closest target but relies on the asymmetry in the parameters of the neural network to commit to the target present in the second position of the observation vector. Whether the robot favors targets in the first, second or third position of the observation vector seems to be only dependent on the random initialization of the parameters of the MLP, since the bias is not corrected at any point of the training process. To enforce inductive bias removal, we next trained under the same exact setting but limiting the maximum speed of the robot to $0.3$m/s. As observed in Fig. \ref{fig:analysis_rl}, the inductive bias is partially corrected after 2 million training steps. However, this is achieved by modifying the dynamics of the robot, and therefore it is not a systematic methodology generalizable across environments and robot embodiments. Alternatively, we trained the same exact setting again but, this time, we introduced a random shuffling in the order of appearance of the targets in the observation vector at each time step of the rollout. Thus, we are explicitly enforcing the policy to not rely on the positioning of the target in the observation vector but simply on relative distance measures. This explicit regularization proves much more effective in removing inductive biases, as displayed in Fig. \ref{fig:analysis_rl}. Nonetheless, this approach has a main caveat: it is not intuitive, since it only emerged as an option because we explicitly tried to understand the decision-making process of the policy. Overall, this raises questions about the actual causality of decision-making in standard neural network parameterizations in RL settings. Perfect removal of inductive biases in the neural network is unfeasible unless explicit symmetry is encoded within the parameters of the network. Even the last regularization alternative was incapable of achieving M1=1. This explains why the RL policy, presented with symmetrical situations in Fig. 2, always commits to the same target from the initial step of the rollout, even under the influence of noise. Then, no decision is actually taken, as the RL policy cannot process all the targets equally and evenly.

\begin{figure}
    \centering
    \begin{tabular}{cc}
         (a) & \includegraphics[width=0.9\linewidth, valign=m]{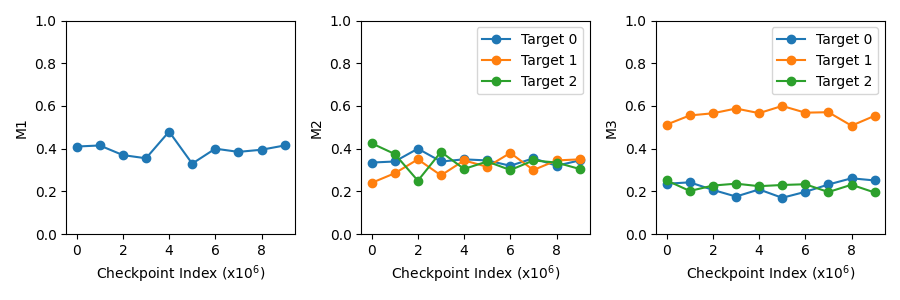} 
         \\
         (b) & \includegraphics[width=0.9\linewidth, valign=m]{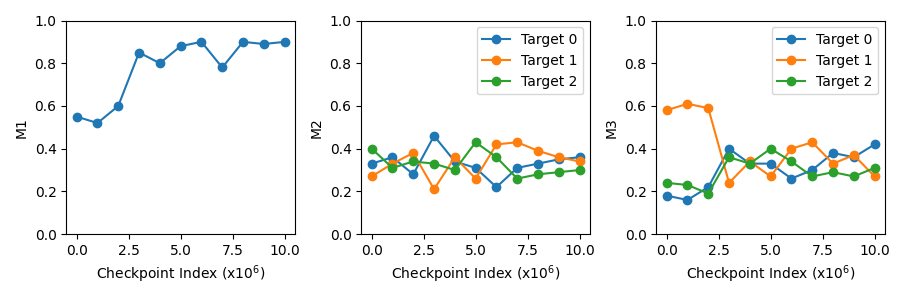}
         \\
         (c) & \includegraphics[width=0.9\linewidth, valign=m]{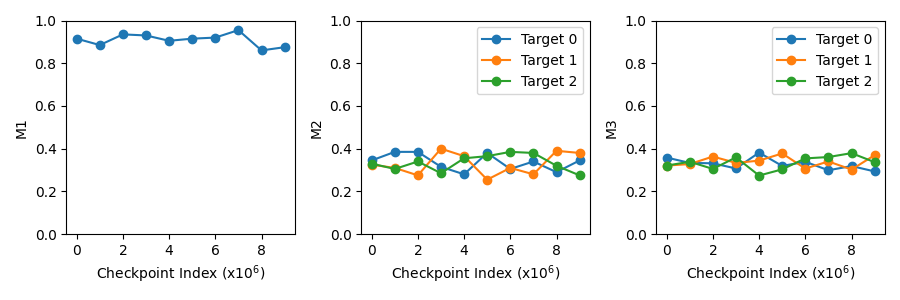}
    \end{tabular}
    \caption{\textbf{Statistical analysis of RL policies}. All the RL policies are trained as specified in Sec. \ref{subsec:rl}. \textbf{a} The policy has unbounded velocity and fixed order of targets in the observation vector. \textbf{b} The policy has bounded velocity of $0.1$m/s and the order of the targets in the observation vector is fixed. \textbf{c} The policy has bounded velocity of $0.3$m/s and the order of the targets in the observation vector is randomized at each time step. The nominal approach in \textbf{a} suffers from inductive biases that bootstrap decision-making and understanding of the optimal way of solving the task at hand i.e., moving to the target present in the $i$-th position in the observation vector instead of the initially closest target. Meanwhile, both regularization methods (limiting the dynamics of the robot and shuffling the order of the targets in the observation vector) mitigate the bias, at the expense of adding counterintuitive modifications in the dynamics of the robot and how it perceives the targets.}
    \label{fig:analysis_rl}
\end{figure}

\section{Invariance of Simplex Under Controlled Neural Dynamics}

In this section we present a proof that for any admissible input $\mathbf{u}(t)$, such as a stream of processed images from a camera, the $k$-dimensional neural decision-making dynamics are constrained to evolve on the unit simplex $\Delta_{k-1}$. We consider a generalized neural decision model as a nonlinear control system of the form
\begin{equation}
    \frac{d\mathbf{n}}{dt} = g(\mathbf{n},\mathbf{u}) := \Pi_{\Delta}(f(\mathbf{n},\mathbf{u})), \quad f(\mathbf{n},\mathbf{u})= - \mathbf{n} +S\big(W(\mathbf{u})\mathbf{n}\big), \label{eq:ndt_vector_1}
\end{equation} 
where the input $\mathbf{u}:[0,\infty) \to U \subset \mathbb{R}^m$ with $U$ nonempty and compact, the input-to-connectivity map $W: U \to \mathbb{R}^{k\times k}$ is continuous, the neural activation function $S: \mathbb{R} \to (0,\infty)$ is strictly positive and locally Lipschitz, and the simplex map $\Pi_{\Delta}$ is defined the same as in the main text as
\begin{equation}
    \Pi_{\Delta}(f(\mathbf{n},\mathbf{u})) = f(\mathbf{n},\mathbf{u}) - \left(\mathbf{1}^T f(\mathbf{n},\mathbf{u})\right) \mathbf{n}. \label{eq:simplex_projection_1}
\end{equation}
The model used in the main text is a special case of the controlled dynamics \eqref{eq:ndt_vector_1} with a specific choice of $W$ and $S$. In the following Theorem we state the invariance result for the general system \eqref{eq:ndt_vector_1}.

\begin{theorem}
    For the neural dynamics \eqref{eq:ndt_vector_1} under the stated assumptions, given any piecewise continuous input $\mathbf{u}: [0,\infty) \to U \subset \mathbb{R}^m$, and any initial condition $\mathbf{n}(0) \in \Delta_{k-1}$, every corresponding solution $\mathbf{n}(t)$ is unique and satisfies $\mathbf{n}(t) \in \Delta_{k-1}$ for all $t \geq 0$.
\end{theorem}

\begin{proof} 

    First, we show that the sum of all neural variables is one for all time.  Consider the quantity $\lambda(t) = \mathbf{1}^T\mathbf{n}(t) $. Its time derivatives reads \begin{equation}\frac{d\lambda}{dt} = \mathbf{1}^T g(\mathbf{n},\mathbf{u}) = \mathbf{1}^T f(\mathbf{n},\mathbf{u}) - (\mathbf{1}^T f(\mathbf{n},\mathbf{u}))(\mathbf{1}^T \mathbf{n}) = (\mathbf{1}^T f(\mathbf{n},\mathbf{u})) (1 - \lambda(t)). \label{eq:lambda_neural}\end{equation}
    To show that this implies $\mathbf{1}^T\mathbf{n} = 1$ for all trajectories $\mathbf{n}(t)$ of the dynamics, consider a specific trajectory $n(t)$ corresponding to a signal $u(t)$ in the neural dynamics, and let  \mbox{$h(t) = \mathbf{1}^Tf(\mathbf{n}(t),\mathbf{u}(t))$}. Under the theorem assumptions, $h \in L^1_{\text{loc}}[0,\infty)$ and $\lambda(t)$ satisfies the scalar linear ODE 
    \begin{equation}
        \frac{d\lambda}{dt} = h(t) (1 - \lambda(t)). \label{eq:lambda_ODE}
    \end{equation}
   Then the differential equation \eqref{eq:lambda_ODE} admits an explicit unique solution for the initial value problem $\lambda(0) = \lambda_0$ as
    \begin{equation}
        \lambda(t) = 1 + (\lambda_0 - 1) \exp\left( -\int_0^t h(s) ds \right). \label{eq:lambdat_sol}
    \end{equation}
Whenever $\lambda_0 = 1$, \eqref{eq:lambdat_sol} simplifies to $\lambda(t) = 1$ for all $t \geq 0$. 
This implies $\mathbf{1}^T \mathbf{n} = 1$ for all $t \geq 0$ for the neural decision dynamics.

      Next, we will show that the set $\Delta_{k-1} =\{\mathbf{n} \in \mathbb{R}^k:  n_i \geq 0, \ \sum_{\ell=1}^k n_{\ell} = 1 \}$
    is robustly positively invariant under the time-varying neural dynamics. To do this, we will show that conditions of the generalized Nagumo theorem are satisfied \cite[Theorem 4.11]{blanchini2008set} for the dynamics restricted to the hyperplane  $S = \{\mathbf{n}: \sum_{i=1}^k n_i = 1\}$. Observe that whenever $n_i = 0$, $$\frac{d n_i}{dt} = S((W(\mathbf{u})\mathbf{n})_i)) > 0$$ for any $\mathbf{u}(t) \in \mathbb{R}^m$. Analogously, whenever $n_i = 1$, $\frac{d n_i}{dt} = -1 + S((W(\mathbf{u})\mathbf{n})_i)) + 1 - \mathbf{1}^T S(W(\mathbf{u}) \mathbf{n}) = - \sum_{j \neq i} S((W(\mathbf{u}) \mathbf{n})_j) < 0$ for any $u(t) \in \mathbb{R}^m$. Then at all boundary points of $\Delta_{k-1}$, the neural state evolves into the set interior and the theorem statement follows.
\end{proof}

\section{Bifurcation Analysis for Two-Option Decision Model}

We study the two input-cluster coarse-grained model (see Methods) which, under the simplex constraint $\bar{n}_2 = 1 - \bar{n}_1$, simplifies to a scalar ordinary differential equation
 \begin{equation}
    \frac{d\bar{n}_1}{dt} = |C_1|(1-\bar{n}_1) S\big((1- c_{12})\bar{n}_1  + c_{12}\big) - |C_2|\bar{n}_1 S\big((c_{12}-1)  \bar{n}_1 + 1\big)
    \label{eq:pitchfork_unfold_n1_1}
\end{equation}
Furthermore, we define the parameters $C,C_{avg}$ such that $C_{avg} \geq  C \geq 0$, with 
\begin{equation}
    |C_1| = C_{\mathrm{avg}} + C,\quad |C_2| = C_{\mathrm{avg}} - C.
\end{equation}
The case $C = 0$ corresponds to  two equally sized clusters, whereas $C>0 \ (<0)$ represents larger (and therefore more preferred) input cluster 1 (cluster 2).
For the purposes of analysis, we define the
shifted coordinate $x=2\bar{n}_1 -1$, $x\in [-1,1]$ and using Eq.~\eqref{eq:pitchfork_unfold_n1_1}  and the observation that $\frac{dx}{dt} = 2 \frac{dn}{dt}$ write down the equivalent system
\begin{equation}
    \frac{dx}{dt} = (1-x)(1 + \bar{C}) S\left(\frac{\mu}{2}x+ \frac{2-\mu}{2}\right) - (1+x) (1 - \bar{C}) S\left(-\frac{\mu}{2} x+\frac{2-\mu}{2}\right)
    \label{eq:pitchfork_x}
\end{equation}
where $\mu = 1-c_{12} \in [0,2]$, and without loss of generality $\bar{C} = C/C_{avg}$ and $a = C_{avg}$ is the upper bound of the saturation $S$. The parameter $\bar{C} \in [-1,1]$ then measures the normalized deviation from symmetry in the inputs. In the following derivations, we denote the right-hand side of \eqref{eq:pitchfork_x} as $ f(x,\mu,\bar{C})$. Observe that the saturation function can equivalently be rewritten as
\begin{equation}
    S(y) = \frac{a}{1 + e^{-\alpha y}} = \frac{a}{2} \left( 1 + \tanh\left(\frac{\alpha y}{2}\right) \right).
\end{equation}
Then, introducing a scaled slope parameter $\tilde{\alpha} = \alpha/4$, \eqref{eq:pitchfork_x} is equivalently expressed as 
\begin{multline}
    \frac{dx}{dt} = - a x + \frac{a}{2} (1-x)(1+\bar{C}) \tanh\left( \tilde{\alpha}\left( \mu x + 2 - \mu \right) \right) \\
    + \frac{a}{2}(1+x) (1 - \bar{C})\tanh\left( \tilde{\alpha} \left( \mu x - 2 + \mu \right) \right) + a \bar{C}.  \label{eq:pitchfork_x_tanh}
\end{multline}

\subsection{Symmetric Inputs $|C_1| = |C_2|$}

From the original dynamics \eqref{eq:pitchfork_unfold_n1_1}, we observe that the compromise state $\bar{n}_1 = \bar{n}_2 = \frac{1}{2}$ is always an equilibrium whenever $|C_1| = |C_2|$, for all values of $c_{12}, a, \alpha$. This is easily verified by substitution, since at $\bar{n}_1 = \frac{1}{2}$ the derivative evaluates to 
\begin{equation}
    \frac{d \bar{n}_1}{dt} = \frac{1}{2} |C_{1}| S\left(\frac{1}{2}(1 - c_{12})\right) - \frac{1}{2} |C_1|  S\left(\frac{1}{2}(1 - c_{12})\right) = 0. 
\end{equation}
In the shifted coordinate system of \eqref{eq:pitchfork_x}, the compromise equilibrium corresponds to the zero equilibrium $x=0$. We study the local stability of $x=0$ with $\bar{C} = 0$. For this symmetric case, the dynamical system \eqref{eq:pitchfork_x_tanh} simplifies to 
\begin{equation}
    \frac{dx}{dt} = - a x + \frac{a}{2} (1-x) \tanh\left(  \tilde{\alpha}\left( \mu x + 2 - \mu \right) \right) + \frac{a}{2} (1+x) \tanh\left( \tilde{\alpha}\left( \mu x - 2 + \mu \right) \right). \label{eq:pitchfork_x_tanh_symmetric}
\end{equation}
To verify $x=0$ is an equilibrium, plugging in and leveraging odd symmetry of the hyperbolic tangent we see that 
\begin{equation}
    \frac{dx}{dt} = \frac{a}{2}\tanh( \tilde{\alpha}(2-\mu)) + \frac{a}{2} \tanh( \tilde{\alpha}(-2 + \mu)) = 0.
\end{equation}
To establish the existence of a bifurcation point
we compute the Jacobian of \eqref{eq:pitchfork_x_tanh_symmetric}, 
\begin{multline}
    J(x;\mu) = \frac{\partial f}{\partial x}(x,\mu,0) = - a + \frac{a}{2} \Big( -\tanh\left( \tilde{\alpha}\left( \mu x + 2 - \mu \right) \right) + \tanh\left( \tilde{\alpha}\left( \mu x - 2 + \mu \right) \right) \Big) \\
     +\frac{a \tilde{\alpha} \mu}{2} \Big( (1-x) \operatorname{sech}^2\left( \tilde{\alpha}\left( \mu x + 2 - \mu \right) \right) + (1+x) \operatorname{sech}^2\left( \tilde{\alpha}\left( \mu x - 2 + \mu \right) \right) \Big).\label{eq:1d_jac}
\end{multline}
The $x = 0$ is locally exponentially stable whenever $J(0,\mu)<0$ and unstable whenever $J(0,\mu) > 0$. A bifurcation point of the compromise equilibrium is a value of the parameter $\mu = \mu^*$ for which $J(0,\mu^*) = 0$. Evaluating \eqref{eq:1d_jac} at $x = 0$ we get
\begin{equation}
    J(0;\mu) = - a - a \tanh( \tilde{\alpha}(2-\mu)) + a \tilde{\alpha} \mu \operatorname{sech}^2\left(  \tilde{\alpha}\left(  2 - \mu \right) \right) 
\end{equation}
which leveraged the fact that $\tanh(-y) = -\tanh(y)$ and $\operatorname{sech}^2(-y) = \operatorname{sech}^2(y)$. Recalling the trigonometric identity $\operatorname{sech}^2(y) = 1 - \tanh^2(y)$, we can equivalently write the above as 
\begin{equation}
    J(0;\mu) = - a - a \tanh( \tilde{\alpha}(2-\mu)) + a  \tilde{\alpha} \mu (1 - \tanh^2( \tilde{\alpha}(2-\mu))). 
\end{equation}
Setting this form of the Jacobian to zero and solving for $\tanh(\alpha(2-\mu))$ using the standard quadratic formula, we find that the bifurcation condition $J(0;\mu^*) = 0$ is equivalent to the condition 
\begin{equation}
    h(\mu^*, \tilde{\alpha}) =  \tilde{\alpha} \mu^* \tanh( \tilde{\alpha}(2 - \mu^*)) -  \tilde{\alpha} \mu^* + 1 = 0. \label{eq:bif_condition_1}
\end{equation}
Observe that $h(0, \tilde{\alpha}) = 1$ and $ h(2, \tilde{\alpha}) = 1 - 2  \tilde{\alpha}$. Furthermore, 
\begin{equation}
    \frac{\partial h}{\partial \mu^*} =  \tilde{\alpha} \Big( \tanh( \tilde{\alpha}(2 - \mu^*)) - 1 -  \tilde{\alpha} \mu^* \operatorname{sech}^2 ( \tilde{\alpha}(2 - \mu^*)) \Big) < 0
\end{equation}
for all $\mu^* \in [0,2]$, $ \tilde{\alpha} > 0$. Therefore the function $h$ is monotonically decreasing over the interval $\mu^* \in [0,2]$ and by the intermediate value theorem it takes on exactly one zero in the interval if $ \tilde{\alpha} > \frac{1}{2}$ (i.e. $\alpha > 2$) and no zeros if $ \tilde{\alpha} < \frac{1}{2}$ (i.e. $\alpha < 2$). This means that whenever $\alpha > \frac{1}{2}$, the neural decision dynamics has a unique a bifurcation point $\mu^*$ defined implicitly by \eqref{eq:bif_condition_1} at which the compromise equilibrium loses stability. Furthermore, observe that
\begin{multline}
    \frac{\partial h}{\partial  \tilde{\alpha}} = \mu^* \tanh( \tilde{\alpha}(2-\mu^*)) - \mu^* + \alpha \mu^* (2 - \mu^*) \operatorname{sech}^2 ( \tilde{\alpha}(2 - \mu^*)) \Big) \\
    = -\frac{1}{ \tilde{\alpha}} + (2 - \mu^*)\left( 2 - \frac{1}{ \tilde{\alpha} \mu^*} \right) = \frac{1}{ \tilde{\alpha}\mu^*}(- \mu^* + (2 - \mu^*) (2  \tilde{\alpha} \mu^* - 1)). \label{eq:dhdalpha}
\end{multline}
At the bifurcation point characterized by \eqref{eq:bif_condition_1}, $\mu^*$ is implicitly defined by the slope parameter $\alpha$ of the neural saturation function, and we find from the bifurcation condition,
\begin{equation}
    \frac{\partial h}{\partial \mu^*} \frac{\partial \mu^*}{\partial  \tilde{\alpha}} + \frac{\partial h}{ \partial  \tilde{\alpha} } = 0 \implies \frac{\partial \mu^*}{\partial  \tilde{\alpha}} = - \frac{\partial h}{ \partial  \tilde{\alpha}} \left( \frac{\partial h}{\partial \mu^*}\right)^{-1}. 
\end{equation}
Since $\frac{\partial h}{\partial \mu^*}$ is always negative, any zeros of the sensitivity function $\frac{\partial \mu^*}{\partial  \tilde{\alpha}} $ must be zeros of $\frac{\partial h}{\partial \alpha }$ in \eqref{eq:dhdalpha}, which in turn correspond to values of $ \tilde{\alpha}, \mu^*$ satisfying $ \tilde{\alpha} = \frac{1}{\mu^* (2 - \mu^*)}$. Substituting this local minimum condition into the bifurcation expression, we find that the smallest possible bifurcation point value $\mu^* = \mu^*_{min}$ is implicitly defined as
\begin{equation}
    \tanh\left(\frac{1}{\mu^*_{min}}\right) + 1 - \mu^*_{min} = 0
\end{equation}
which evaluates to approximately $\mu^*_{min} \approx 1.5644$ at the slope $ \tilde{\alpha}^*_{min} \approx 1.4674 (\alpha_{min}^* \approx 5.8696)$. Since $\mu = 1 - \cos(\theta_{12})$ where $\theta_{12}$ is the critical angle at which the agent makes a decision between two options, this means for the coarse-grained two-option model, the smallest possible angle at which a bifurcation happens is around $\theta_{min} \approx 124^{\circ}$. 

Finally, to complete the analysis we will classify the bifurcation by expanding \eqref{eq:pitchfork_x_tanh_symmetric} in $x$ and $\mu$ around the bifurcation point $(x,\mu,\bar{C}) = (0,\mu^*,0)$ with $\mu^*$ implicitly defined by \eqref{eq:bif_condition_1}, with $\bar{C} = 0$, i.e. approximating the right-hand side of the equation with low-order terms
 \begin{equation}
    \begin{aligned}
        f(x,\mu,0)
        &= f(0,\mu^{*},0)
         + f_{x}(0,\mu^{*},0)\,x
         + f_{\mu}(0,\mu^{*},0)\,(\mu-\mu^{*})
        \\
        &\qquad
         + \frac12 f_{xx}(0,\mu^{*},0)\,x^2
         + f_{x\mu}(0,\mu^{*},0)\,x(\mu-\mu^{*})
        \\
        &\qquad
         + \frac16 f_{xxx}(0,\mu^{*},0)\,x^3
         + \mathcal{O}\!\left(|x|^4 + |x|^2|\mu-\mu^{*}|\right).
    \end{aligned}
    \label{eq:taylor_sym}
    \end{equation}
    The odd symmetry of $f(x,\mu,\bar{C})$ in $x$ and the equilibrium conditions imply $ f(0,\mu^{*},0) = 0,
    f_{\mu}(0,\mu^{*},0) = 0$, and $f_{xx}(0,\mu^{*}) = 0$. Taking derivatives of the  model \eqref{eq:pitchfork_x_tanh_symmetric} shows that
    \[
        f_{x\mu}(0,\mu^{*},0) \neq 0,
        \qquad
        f_{xxx}(0,\mu^{*},0) \neq 0.
    \]
    Thus we get the simplified expansion
    \[
    f(x,\mu,0)=f_{x\mu}(0,\mu^{*},0)(\mu-\mu^{*})x +  \frac{1}{6}f_{xxx}(0,\mu^{*},0)x^3 + \mathcal{O}\!\left(|x|^4 + |x|^2|\mu-\mu^{*}|\right).
    \]
    where 
    \begin{equation}
        f_{x\mu}(0,\mu^*,0) = 2 a  \tilde{\alpha} \operatorname{sech}^2(\tilde{\alpha} (2 -\mu^*)) \big(1 +  \tilde{\alpha} \mu^* \tanh( \tilde{\alpha} (2 - \mu^*))\big) > 0, 
    \end{equation}
and
    \begin{multline}
        f_{xxx}(0,\mu^*,0) = 2 a  \tilde{\alpha}^2 (\mu^*)^2 \operatorname{sech}^2(\tilde{\alpha} (2 -  \mu^*)) \\
        \cdot \left( 3 \tanh( \tilde{\alpha} (2 - \mu^*)) +  \tilde{\alpha} \mu^* \left( 3 \tanh^2 ( \tilde{\alpha}(2 - \mu^*)) - 1 \right) \right). 
    \end{multline}
    As long as the nondegeneracy condition $f_{xxx}(0,\mu^*,0) \neq 0$, i.e. $\left( 3 \tanh(\tilde{\alpha}(2 - \mu^*)) +\tilde{\alpha} \mu^* \left( 3 \tanh^2 ( \tilde{\alpha} (2 - \mu^*)) - 1 \right) \right) \neq 0$ is satisfied, we can conclude that the bifurcation at $(x,\mu,\bar{C}) = (0,\mu^*,0)$ is a pitchfork bifurcation using standard classification results from the singularity theory of bifurcations \cite[Proposition 9.1]{Golubitsky1985}. When $\left( 3 \tanh(\tilde{\alpha} (2 - \mu^*)) + \tilde{\alpha}a \mu^* \left( 3 \tanh^2 ( \tilde{\alpha} (2 - \mu^*)) - 1 \right) \right) < 0 (> 0)$, the bifurcation is supercritical (subcritical), i.e. as $x = 0$ loses stability, two new symmetric equilibria $x = \pm x^*$ branch off from the origin in a neighborhood of $(x,\mu) = (0.\mu^*)$, with the branches first appearing for $\mu > \mu^*$ in the supercritical case and for $\mu < \mu^*$ in the subcritical case. The positive equilibrium $x = x^* > 0$ corresponds to a decision in favor of option 1, represented by input cluster $C_1$, and the negative equilibrium $x = - x^*$ corresponds to a decision in favor of option 2, represented by input cluster $C_2$. This analysis illustrates how the decision process is able to ``flip a coin'' and make a choice in the presence of symmetric options.

\subsection{Asymmetric Inputs $|C_1| \neq |C_2|$}
So far we established that a symmetric pitchfork bifurcation organizes the neural decision process when presented with two symmetric targets. In the asymmetric case, the normalized parameter $\bar{C}$ in \eqref{eq:pitchfork_x_tanh} is nonzero and the equation no longer has an odd symmetry. For the asymmetric model, the imbalance between $|C_1|$ and $|C_2|$ implies generically $f_{\bar{C}}(0,\mu^{*},0) \neq 0$ and an the low-order terms of the expansion of the system about the bifurcation point $(x,\mu,\bar{C}) = (0,\mu^*,0)$ become
    \[
        f(x,\mu,\bar{C}) = \beta_1 (\mu-\mu^{*})x + \beta_2 x^3 + \beta_3 \bar{C} +\mathcal{O}\left(|x|^4 + |x|^2|\mu-\mu^{*}| + |x||C|\right), 
    \]
    where the nonzero low-order parameters are $\beta_1 = f_{x\mu}(0,\mu^{*},0) $, $\beta_2 = \frac{1}{6}f_{xxx}(0,\mu^{*},0) $ as in the symmetric case, and  $\beta_3 = f_{\bar{C}}(0,\mu^{*},0)$. To compute $\beta_3$ we take the derivative of \eqref{eq:pitchfork_x_tanh} with respect to $\bar{C}$, 
    \begin{equation}
        \frac{\partial f}{\partial \bar{C}} = \frac{a}{2} (1 - x) \tanh( \tilde{\alpha} (\mu x + 2 - \mu) ) - \frac{a}{2}(1 + x) \tanh( \tilde{\alpha} (\mu x - 2 +\mu) ) + a
    \end{equation}
    and at the bifurcation point this evaluates to 
    \begin{equation}
        f_{\bar{C}}(0,\mu^*,0) = a \big(\tanh(\tilde{\alpha} (2 - \mu^*)) + 1\big) > 0.
    \end{equation}
    The parameter $\bar{C}$ acts as an \textit{unfolding parameter} for the symmetric pitchfork bifurcation, breaking up the symmetric diagram in a neighborhood of its bifurcation point. For a comprehensive development of unfolding theory of bifurcations we refer the interested reader to \cite{Golubitsky1985}, with the unfolding of the pitchfork developed in Chapter I of the reference. In particular, since the expansion coefficient is positive, in a neighborhood of the bifurcation point $(x,\mu,\bar{C}) = (0,\mu^*,0)$, we expect the symmetric bifurcation diagram to break up into a continuous branch of equilibria with a larger region of attraction and a secondary branch that appears in a saddle-node bifurcation. The favored equilibrium along the continuous branch will correspond to a choice for option 1 if $\bar{C} > 0$, i.e. if input cluster 1 is bigger than input cluster 2, and conversely it will correspond to a choice for option 2 if $\bar{C} < 0$ and input cluster 2 is bigger. In this way, the decision dynamics encode ultrasensitivity to input difference, since even small mismatch between $|C_1|$ and $|C_2|$ will bias the decision in a neighborhood of the bifurcation point. Since agents moving through space have the parameter $\mu$ evolve, increasing from a pre-bifurcation value to a post-bifurcation value, the decision dynamics will pass through a neighborhood of the bifurcation point and amplify small differences between cluster inputs. Bifurcation diagrams illustrating the symmetric and asymmetric equilibrium branches in the coarse-grained model are shown in Figure 7 of the main document.

\section*{Legends}
\begin{itemize}
 
    \item \textbf{Supplementary Video 1} \textbf{Neuromorphic control for autonomous navigation and tracking.} Overview of the proposed neuromorphic control framework, which converts raw camera pixels into neural population dynamics that directly produce egocentric motion commands. Hardware experiments using an onboard-compute quadrotor demonstrate robust and repeatable decision-making in symmetric two-goal environments, even in the presence of noise and real-world disturbances. Meanwhile, AirGen simulations showcase multi-goal navigation and tracking tasks, with synchronized multi-view visualizations that reveal how neural activity evolves in real time in response to visual input.

    \item \textbf{Supplementary Video 2} \textbf{AirGen simulation of multi-target navigation in blocks environment.} First-person camera views, target-evidence maps, neural activity states, and thresholded neural representations are shown throughout a representative navigation run.
    \item \textbf{Supplementary Video 3} \textbf{AirGen simulation of dynamic multi-target tracking.} First-person camera views, target-evidence maps, neural activity states, and thresholded neural representations are shown throughout a representative navigation run.
    
\end{itemize}

\printbibliography[title={Supplementary References}]
\end{refsection}

\end{document}